\documentclass{article}

\PassOptionsToPackage{inline}{enumitem}

\usepackage{microtype}
\usepackage{graphicx}
\usepackage{subcaption}
\usepackage{booktabs} 

\usepackage[T1]{fontenc}
\usepackage{hyperref}

\usepackage{amsmath}
\usepackage{amsfonts}
\usepackage{amssymb}
\usepackage{amsthm}
\usepackage{bm}
\usepackage{mathtools}
\usepackage{enumitem}
\usepackage{thmtools,thm-restate}
\usepackage{algorithm}
\usepackage{algorithmic}
\usepackage{natbib}
\usepackage[capitalise]{cleveref}
\usepackage[dvipsnames]{xcolor}
\usepackage{graphicx}
\usepackage{comment}

\theoremstyle{plain}
\newtheorem{lemma}{Lemma}
\newtheorem{proposition}{Proposition}
\newtheorem{corollary}{Corollary}

\theoremstyle{definition}
\newtheorem{definition}{Definition}

\theoremstyle{remark}

\def\AA{\mathcal{A}}\def\BB{\mathcal{B}}
\def\DD{\mathcal{D}}\def\FF{\mathcal{F}}
\def\GG{\mathcal{G}}\def\II{\mathcal{I}}

\def\MM{\mathcal{M}}
\def\PP{\mathcal{P}}
\def\SS{\mathcal{S}}\def\TT{\mathcal{T}}
\def\XX{\mathcal{X}}

\def\Ebb{\mathbb{E}}

\def\Rbb{\mathbb{R}}

\def\R{\Rbb}

\usepackage{dsfont}
\def\one{{\mathds{1}}}

\def\*{\star}
\DeclareMathSymbol{\mhef}{\mathord}{operators}{`\-}

\newcommand{\norm}[1]{ \| #1 \|  }
\newcommand{\abs}[1]{ \left| #1 \right|  }

\DeclareMathOperator*{\argmin}{arg\,min}

\newcommand{\E}{\Ebb}

\usepackage{algorithmic}
\renewcommand{\algorithmicrequire}{\textbf{Input:}}
\renewcommand{\algorithmicensure}{\textbf{Output:}}
\usepackage{xspace}
\usepackage{tikz}

\usepackage[accepted]{icml2021}

\icmltitlerunning{Safe Reinforcement Learning Using Advantage-Based Intervention}

\def\prob{\mathrm{Prob}}
\def\safeset{\ensuremath{\SS_{\mathrm{safe}}}}
\def\unsafeset{\ensuremath{\SS_{\mathrm{unsafe}}}}
\def\violation{\ensuremath{s_\triangleright}}
\def\absorbing{\ensuremath{s_\circ}}
\def\intervened{\ensuremath{s_\dagger}}
\def\supp{\mathrm{Supp}}
\def\alg{\textsc{SAILR}\xspace}
\def\algfull{Safe Advantage-based Intervention for Learning policies with Reinforcement\xspace}

\usepackage{inconsolata}

\begin{document}

\twocolumn[
\icmltitle{Safe Reinforcement Learning Using Advantage-Based Intervention}




\begin{icmlauthorlist}
    \icmlauthor{Nolan Wagener}{GT}
    \icmlauthor{Byron Boots}{UW}
    \icmlauthor{Ching-An Cheng}{MSR}
\end{icmlauthorlist}

\icmlaffiliation{GT}{Institute for Robotics and Intelligent Machines, Georgia Institute of Technology, Atlanta, Georgia, USA}
\icmlaffiliation{UW}{Paul G. Allen School of Computer Science and Engineering, University of Washington, Seattle, Washington, USA}
\icmlaffiliation{MSR}{Microsoft Research, Redmond, Washington, USA}

\icmlcorrespondingauthor{Nolan Wagener}{\href{mailto:nolan.wagener@gatech.edu}{\texttt{nolan.wagener@gatech.edu}}}

\icmlkeywords{Reinforcement Learning Theory, Safe Reinforcement Learning}

\vskip 0.3in
]



\printAffiliationsAndNotice{}  

\begin{abstract}

Many sequential decision problems involve finding a policy that maximizes total reward while obeying safety constraints.
Although much recent research has focused on the development of safe reinforcement learning (RL) algorithms that produce a safe policy {after} training, ensuring safety \emph{during} training as well remains an open problem.
A fundamental challenge is performing exploration while still satisfying constraints in an unknown Markov decision process (MDP).
In this work, we address this problem for the chance-constrained setting.
We propose a new algorithm, \alg, that uses an {intervention mechanism} based on advantage functions to keep the agent safe throughout training and optimizes the agent's policy using off-the-shelf RL algorithms designed for unconstrained MDPs.
Our method comes with strong guarantees on safety during \emph{both} training and deployment (i.e., after training and without the intervention mechanism) and policy performance compared to the optimal safety-constrained policy.
In our experiments, we show that \alg violates constraints far less during training than standard safe RL and constrained MDP approaches and converges to a well-performing policy that can be deployed safely without intervention.
Our code is available at \href{https://github.com/nolanwagener/safe_rl}{\texttt{https://github.com/nolanwagener/safe\_rl}}.
\end{abstract}

\section{Introduction}
\begin{figure}[t]
    \centering
    \includegraphics[width=0.4\textwidth]{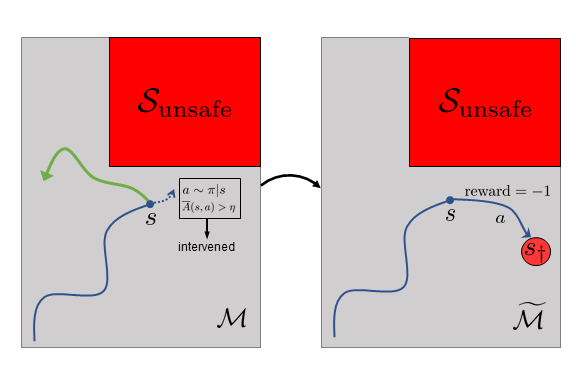}
    \caption{Advantage-based intervention of \alg and construction of the surrogate MDP $\widetilde\MM$. In $\MM$, whenever the policy $\pi$ proposes an action $a$ which is disadvantageous (w.r.t. a backup policy $\mu$) in terms of safety, $\mu$ intervenes and guides the agent to safety (green path). From the perspective of $\pi$, it transitions to an absorbing state $\intervened$ and receives a penalizing reward of $-1$.}
    \label{fig:intervention}
\end{figure}
Reinforcement learning (RL) \citep{sutton2018reinforcement} enables an agent to learn good behaviors with high returns through interactions with an environment of interest.
However, in many settings, we want the agent not only to find a high-return policy but also avoid undesirable states as much as possible, even during training.
For example, in a bipedal locomotion task, we do not want the robot to fall over and risk damaging itself either during training or deployment.
Maintaining safety while exploring an unknown environment is challenging, because venturing into new regions of the state space may carry a chance of a costly failure.

Safe reinforcement learning \citep{garcia2015comprehensive, amodei2016concrete} studies the problem of designing learning agents for sequential decision-making with this challenge in mind.
Most safe RL approaches tackle the safety requirement either by framing the problem as a {constrained Markov decision process (CMDP)} \citep{altman1999constrained} or by using control-theoretic tools to restrict the actions that the learner can take.
However, due to the natural conflict between learning, maximizing long-term reward, and satisfying safety constraints, these approaches make different performance trade-offs.

CMDP-based approaches~\citep{borkar2005actor,achiam2017constrained,le2019batch} take inspiration from existing constrained optimization algorithms for non-sequential problems, notably the Lagrangian method~\citep{bertsekas2014constrained}.
The most prominent examples~\citep{chow2017risk,tessler2018reward} rely on first-order primal-dual optimization to solve a stochastic nonconvex saddle-point problem.
Though they eventually produce a safe policy, such approaches have no guarantees on policy safety during training.
Other safe RL approaches~\citep{achiam2017constrained,le2019batch,bharadhwaj2021conservative} conservatively enforce safety constraints on every policy iterate by solving a constrained optimization problem, but they can be difficult to scale due to their high computational complexity.
All of the above methods suffer from numerical instability originating in solving the stochastic nonconvex saddle-point problems~\citep{facchinei2007finite,lin2020gradient}; 
consequently, they are less robust than typical unconstrained RL algorithms.

Control-theoretic approaches to safe RL use interventions, projections, or planning \citep{hans2008safe,wabersich2018safe,dalal2018safe,berkenkamp2017safe} to enforce safe interactions between the agent and the environment, independent of the policy the agent uses.
The idea is to use domain-specific heuristics to decide whether an action proposed by the agent's policy can be safely executed.
However, some of these algorithms do not allow the agent to \emph{learn} to be safe after training~\citep{ wabersich2018safe,hans2008safe,polo2011safe}, so they may not be applicable in scenarios where the control mechanism relies on resouces only available during training (such as computationally demanding online planning).
It is also often unclear how these policies perform compared to the optimal policy in the CMDP-based approach.  
In this work, we propose a new algorithm, \alg (\emph{\algfull}), that uses a novel advantage-based intervention rule to enable safe and stable RL for \emph{general} MDPs.
Our method comes with strong guarantees on safety during \emph{both} training and deployment (i.e., after training and without the intervention mechanism) and has good on-policy performance compared to the optimal safety-constrained policy.
Specifically, \alg trains the agent's policy by calling an off-the-shelf RL algorithm designed for standard \emph{unconstrained} MDPs.
In each iteration, \alg:
\begin{enumerate*}[label=\textit{\arabic*)}]
    \item queries the base RL algorithm to get a data-collection policy; 
    \item runs the policy in the MDP while utilizing the advantage-based intervention rule to ensure safe interactions (and executes a backup policy upon intervention to ensure safety);
    \item transforms the collected data into experiences in a new unconstrained MDP that penalizes any visits of intervened state-actions (visualized in~\cref{fig:intervention}); and 
    \item gives the transformed data to the base RL algorithm to perform policy optimization.
\end{enumerate*}

Under very mild assumptions on the MDP and the safety of the backup policy used during the intervention,\footnote{We only assume that the unsafe states are absorbing and that the backup policy is safe from the initial state with high probability. We do \emph{not} assume that the backup policy can achieve high rewards.} we prove that running \alg with \emph{any} RL algorithm for unconstrained MDPs can safely learn a policy that has good performance in the safety-constrained MDP (with a bias propotional to how often the true optimal policy
would be overridden by our intervention mechanism).
Compared with existing work, \alg is easier to implement and runs more reliably than the CMDP-based approaches. In addition, since we only rely on estimated advantage functions, our approach is also more generic than the aforementioned control-theoretic approaches which make assumptions on smoothness or ergodicity of the problem.

We also empirically validate our theory by comparing \alg with several standard safe RL algorithms in simulated robotics tasks. 
The encouraging experimental results strongly support the theory: \alg can learn safe policies with competitive performance using a standard unconstrained RL algorithm, PPO~\citep{schulman2017proximal}, while incurring only a small fraction of unsafe training rollouts compared to the baselines.

\vspace{-1mm}
\section{Preliminaries} \label{sec:problem formulation}

\subsection{Notation}
\vspace{-1mm}

A $\gamma$-discounted infinite-horizon Markov decision process (MDP) is denoted as a 5-tuple, $\MM \coloneqq(\SS,\AA,P, r, \gamma)$, 
where $\SS$ is a state space, $\AA$ is an action space, $P(s'|s,a)$ is a transition dynamics, $r(s,a)\in[0,1]$ is a reward function, and $\gamma\in[0,1)$ is a discount factor.
In this work, $\SS$ and $\AA$ can be either discrete or continuous.
A policy $\pi$ on $\MM$ is a mapping $\pi: \SS \to \Delta(\AA)$, where $\Delta(\AA)$ denotes probability distributions on $\AA$.
We use the following overloaded notation:
For a state distribution $d\in\Delta(\SS)$ and a function ${f : \SS\to\R}$, we define $f(d) \coloneqq \E_{s\sim d}[f(s)]$; similarly, for a policy $\pi$ and a function $g :\SS\times\AA\to\R$, we define ${g(s, \pi) \coloneqq \E_{a\sim\pi|s}[g(s,a)]}$.
We would often omit the random variable in the subscript of expectations, if it is clear.

A policy $\pi$ induces a trajectory distribution $\rho^\pi(\xi)$, where $\xi = (s_0, a_0, s_1, a_1, \dots)$ denotes a random trajectory. The state-action value function of $\pi$ is defined as
${Q^\pi(s,a) \coloneqq \E_{\xi\sim\rho^\pi|s_0=s,a_0=a}[\sum_{t=0}^\infty \gamma^t r(s_t,a_t) ]}$ and its state value function as $V^\pi(s) = Q^\pi(s,\pi)$.
We denote the optimal stationary policy of $\MM$ as $\pi^*$ and its respective value functions as $Q^*$ and $V^*$. 
Let $d_t^\pi(s)$ be the state distribution at time $t$ induced by running $\pi$ in $\MM$ from an initial state distribution $d_0$ (note that $d_0^\pi = d_0$); then the average state distribution induced by $\pi$ is $d^\pi(s) \coloneqq (1-\gamma) \sum_{t=0}^\infty \gamma^t d_t^\pi(s)$.
For brevity, we overload the notation $d^\pi$ to also denote the state-action distribution $d^\pi(s,a) \coloneqq d^\pi(s)\pi(a|s)$.
Finally, later in the paper we will consider multiple variants of an MDP (specifically, $\MM$, $\overline{\MM}$, and $\widetilde{\MM}$) and will use the decorative symbol on the MDP notation to distinguish similar objects from different MDPs (e.g., $V^\pi$ and $\overline{V}^\pi$ will denote the state value functions of $\pi$ in $\MM$ and $\overline{\MM}$).
Throughout this paper, we'll take $\E_{s'|s,a}$ to mean $\E_{s'\sim \PP |s,a}$.

\vspace{-1mm}
\subsection{Safe Reinforcment Learning}
\vspace{-1mm}

We consider safe RL in a $\gamma$-discounted infinite horizon MDP $\MM$,
where safety means that the probability of the agent entering an unsafe subset $\unsafeset \subset \SS$ is low.
We assume that we know the unsafe subset $\unsafeset$ and the safe subset $\safeset\coloneqq \SS\setminus\unsafeset$.
However, we make no assumption on the knowledge of reward $r$ and dynamics $P$, except that the reward $r$ is zero on $\unsafeset$ and that $\unsafeset$ is absorbing: once the agent enters $\unsafeset$ in a rollout, it cannot travel back to $\safeset$ and stays in $\unsafeset$ for the rest of the rollout.

\vspace{-1mm}
\paragraph{Objective}
Our goal is to find a policy $\pi$ that is safe and has a high return in $\MM$, and to do so via a safe data collection process.
Specifically, while keeping the agent safe during exploration, we want to solve the following chance-constrained policy optimization problem:
\begin{align} \label{eq:basic formulation}
    \max_\pi \quad & V^\pi(d_0)\\
    \mathrm{s.t.} \quad &  (1-\gamma) \sum_{h=0}^\infty \gamma^h \prob(\xi_h\subset\safeset \mid \pi) \geq 1-\delta, \nonumber
\end{align}
where $\delta\in[0,1]$ is the tolerated failure probability, $\xi_h=(s_0,a_0,\dots,s_{h-1},a_{h-1})$ denotes an $h$-step trajectory segment, and $\prob(\xi_h\subset\safeset \mid \pi)$ denotes the probability of $\xi_h$ being safe (i.e., not entering $\unsafeset$ from time step $0$ to $h-1$) under the trajectory distribution $\rho^\pi$ of $\pi$ on $\MM$.\footnote{We abuse the notation $\xi_h \subset \safeset$ to mean that $s_\tau \in \safeset$ for each $s_\tau$ in $\xi_h =
(s_0, a_0, \dots, s_{h-1}, a_{h-1})$.}

We desire the the agent to provide anytime safety in \emph{both} training and deployment.
During training, the agent can interact with the unknown MDP $\MM$ to collect data under a training budget, such as the maximum number of environment interactions or allowed unsafe trajectories the agent can generate.
Once the budget is used up, training stops, and an approximate solution of \eqref{eq:basic formulation} needs to be returned.

The constraint in \eqref{eq:basic formulation} is known as a \emph{chance constraint}. The definition here accords to an exponentially weighted average (based on the discount factor $\gamma$) of trajectory safety probabilities of different horizons.
This weighted average concept arises naturally in $\gamma$-discounted MDPs, because the objective in \eqref{eq:basic formulation} can also be written as a weighted average of undiscounted expected returns, i.e.,
$V^\pi(d_0) = (1-\gamma) \sum_{h=0}^\infty \gamma^h U_h^\pi(d_0)$, where $U_h^\pi(d_0) \coloneqq \E_{\rho^\pi}[\sum_{t=0}^h  r(s_t,a_t) ]$.

\vspace{-1mm}
\paragraph{CMDP Formulation}
The chance-constrained policy optimization problem in \eqref{eq:basic formulation} can be formulated as a constrained Markov decision process (CMDP) problem \citep{altman1999constrained, chow2017risk}.
For the mathematical convenience of defining and analyzing the equivalence between \eqref{eq:basic formulation} and a CMDP, instead of treating $\unsafeset$ as a single meta-absorbing state, without loss of generality we define $\unsafeset \coloneqq \{ \violation, \absorbing \}$. The semantics of this set is that when an agent leaves $\safeset$ and enters $\unsafeset$, it first goes to $\violation$ and, regardless of which action it takes at $\violation$, it then goes to the absorbing state $\absorbing$ and stays there forever. We can view $\violation$ as a meta-state that summarizes the unsafe region in a given RL application (e.g., a biped robot falling on the ground) and $\absorbing$ as a fictitious state that captures the absorbing property of $\unsafeset$.

For an MDP $\MM$ with an unsafe set $\unsafeset \coloneqq \{ \violation, \absorbing \}$, define the cost $c(s,a) \coloneqq \one\{ s=\violation \}$, where $\one$ denotes the indicator function. 
Then we can define a CMDP $(\SS, \AA, P, r, c, \gamma)$ using a reward-based MDP $\MM \coloneqq (\SS, \AA, P, {\color{blue} r}, \gamma)$ and a cost-based MDP $\overline{\MM} \coloneqq (\SS, \AA, P, {\color{blue} c}, \gamma)$.
Using these new definitions, we can write the chance-constrained policy optimization in \eqref{eq:basic formulation} as a CMDP problem:
\begin{align}  \label{eq:CMDP formulation}
    \max_\pi \quad & V^\pi(d_0)   \quad  \mathrm{subject~to}  \quad   \overline{V}^\pi(d_0) \leq \delta .
\end{align}
For completeness, we include a proof of this equivalence in \cref{app:proof of CMDP version}, which follows from the fact that the unsafe probability can be represented as the expected cumulative cost, i.e., $\prob(\violation \in \xi^h \mid \pi ) = \E_{\rho^\pi} [ \sum_{t=0}^{h-1} c(s_t,a_t)]$.
In other words, the chance-constrained policy optimization problem is a CMDP problem that aims to find a policy that has a high cumulative reward $V^\pi(d_0)$ with cumulative cost $\overline{V}^\pi(d_0)$ below the allowed failure probability $\delta$.

\vspace{-1mm}
\paragraph{Challenges}
This CMDP formulation has been commonly studied to design RL algorithms to find good polices that can be deployed safely~\citep{chow2017risk, achiam2017constrained,tessler2018reward,efroni2020exploration}.
However, as mentioned in the introduction, these algorithms do not necessarily ensure safety during training and can be numerically unstable.
At a high level, this instability stems from the lack of off-the-shelf computationally reliable and efficient solvers for large-scale constrained stochastic optimization. 

While several control-theoretic techniques have been proposed to ensure safe data collection~\citep{dalal2018safe, wabersich2018safe,perkins2002lyapunov,chow2018lyapunov,chow2019lyapunov,berkenkamp2017safe,fisac2018general} and in some cases prevent the need for solving a constrained problem, it is unclear how the learned policy performs in terms of the objective $V^\pi(d_0)$ in \eqref{eq:CMDP formulation} (i.e., without any interventions).
Most of these algorithms also require stronger assumptions on the environment than approaches based on CMDPs (e.g., smoothness or ergodicity).

As we will show, our proposed approach retains the best of both approaches, ensuring safe data collection via interventions while guaranteeing good performance and safety when deployed without the intervention mechanism.

\vspace{-1mm}

\section{Method} \label{sec:method}
Our safe RL approach, \alg, finds an approximate solution to the CMDP problem in \eqref{eq:CMDP formulation} by using an advantage-based intervention rule for safe data collection and an off-the-shelf RL algorithm for policy optimization.
As we will see, \alg can ensure safety for \emph{both} training and deployment, when 
\begin{enumerate*}[label=\textit{\arabic*)}]
    \item the intervention rule belongs to an ``admissible class'' (see \cref{def:admissible intervention} in \cref{sec:gr}); and
    \item  the base RL algorithm finds a nearly optimal policy for a new \emph{unconstrained} problem of  a surrogate MDP $\widetilde{\MM}$
    constructed by the \emph{intervention rule} together with $\MM$.
\end{enumerate*}
Moreover, because \alg can reuse existing RL algorithms for unconstrained MDPs to optimize policies, it is easier to implement and is more stable than typical CMDP approaches based on constrained optimization.

Specifically, \alg optimizes policies iteratively as outlined in~\cref{alg:ouralg}.
As input, it takes an RL algorithm $\FF$ for unconstrained MDPs and an intervention rule $\GG:\pi\mapsto \GG(\pi)$, where $\pi' = \GG(\pi)$ is a \emph{shielded policy}  such that $\pi'$ runs a \emph{backup policy} $\mu:\SS\to\Delta(\AA)$ instead of $\pi$ when $\pi$ proposes ``unsafe actions.''
In every iteration, \alg first queries the base RL algorithm $\FF$ for a data-collection policy $\pi$ to execute in $\widetilde{\MM}$ (line~\ref{line:get RL policy}).
Then it uses the intervention rule $\GG$ to modify $\pi$ into $\pi'$ (line~\ref{line:get shielded policy}) such that running $\pi'$ in the original MDP $\MM$ can be safe with high probability while effectively simulating execution of $\pi$ in the surrogate $\widetilde{\MM}$.
Next, it collects the data $\DD$ by running $\pi'$ in $\MM$ and then transforms it into new data $\widetilde\DD$ of $\pi$ in $\widetilde\MM$ (line~\ref{line:collect data}).
It then feeds $\widetilde\DD$ to the base RL algorithm $\FF$ for policy optimization (line~\ref{line:optimize}), and optionally uses $\DD$ to refine the intervention rule $\GG$ (line~\ref{line:update intervention}).
The process above is repeated until the training budget is used up.
When this happens, \alg terminates and returns the best policy $\hat{\pi}^*$ the base algorithm $\FF$ can produce for $\widetilde{\MM}$ so far (line~\ref{line:return}). 

We provide the following informal guarantee for \alg, which is a corollary of our main result in \cref{th:main theorem} presented in \cref{sec:main result}.
\begin{proposition}[Informal Guarantee] \label{th:informal theorem}
    For \alg, if the intervention rule $\GG$ is admissible (\cref{def:admissible intervention} in \cref{sec:gr})) and the RL algorithm $\FF$ learns an $\varepsilon$-suboptimal policy $\hat\pi$ for $\widetilde{\MM}$,
    then, for any comparator policy $\pi^*$, $\hat{\pi}$ has the following performance and safety guarantees in $\MM$:
    \begin{align*}
        V^{\pi^*}(d_0) - V^{\hat\pi}(d_0)
        &\leq \frac{2}{1-\gamma} P_\GG(\pi^*) + \varepsilon \\
        \overline{V}^{\hat\pi} (d_0) &\leq \overline{V}^\mu(d_0) + \varepsilon,
    \end{align*}
    where $\mu$ is the backup policy in $\GG$ and $P_{\GG}(\pi^*)$ is the probability that $\pi^*$ visits the intervention set of $\GG$ in $\MM$.
\end{proposition}

In other words, if the base algorithm $\FF$ used by \alg can find an $\varepsilon$-suboptimal policy for the surrogate, unconstrained MDP $\widetilde{\MM}$, then the policy returned by \alg is roughly $\varepsilon$-suboptimal in the original MDP $\MM$, up to an additional error proportional to the probabilty that the comparator policy $\pi^*$ would be overridden by the intervention rule $\GG$ at some point while running in $\MM$.
Furthermore, the returned policy $\hat\pi$ is as safe as the backup policy $\mu$ of the intervention rule $\GG$, up to an additional unsafe probability $\varepsilon$ arising from the suboptimality in solving $\widetilde{\MM}$ with $\FF$.

\begin{algorithm}[t]
    \caption{\alg}
    \begin{algorithmic}[1]
        \label{alg:ouralg}
        \REQUIRE MDP $\MM$, RL algorithm $\FF$, Intervention rule $\GG$
        \ENSURE Optimized safe policy $\hat{\pi}^*$
        \STATE  $\FF\texttt{.Initialize()}$ \label{line:initialize}
        \WHILE{training budget available}
        \STATE $\pi \gets \FF\texttt{.GetDataCollectionPolicy()}$ \label{line:get RL policy}
        \STATE $\pi' \gets \texttt{DefineShieldedPolicy(} \pi \texttt{,} \GG \texttt{)}$ \label{line:get shielded policy}
        \STATE $\DD, \widetilde\DD \gets \texttt{CollectData(} \pi' \texttt{,} \MM \texttt{)}$ \label{line:collect data}
        \STATE $\FF\texttt{.OptimizePolicy(} \widetilde\DD \texttt{)}$ \label{line:optimize}
        \STATE $\GG \gets \texttt{UpdateInterventionRule(} \DD \texttt{)}$ \quad \textit{(Optional)} \label{line:update intervention}
        \ENDWHILE
        \STATE $\hat\pi^* \gets \FF\texttt{.GetOptimizedPolicy()}$ \label{line:return}
    \end{algorithmic}
\end{algorithm}

We point out the results above hold without any assumption on the MDP (other than that the unsafe subset $\unsafeset$ is absorbing and the reward is zero on $\unsafeset$). 
To learn a safe policy, \alg only needs a good \emph{unconstrained} RL algorithm $\FF$, a backup policy $\mu$ that is safe starting at the \emph{initial state} (not globally), and an advantage function estimate of $\mu$, as we explain later in this section.

The price we pay for keeping the agent safe using an intervention rule $\GG$ is a performance bias proportional to $P_{\GG}(\pi^*)/(1-\gamma)$.
This happens because employing an intervention rule during data collection limits where the agent can explore in $\MM$.
Thus, if the comparator policy $\pi^*$ goes to high-reward states which would be cut off by the intervention rule, \alg (and any other intervention-based algorithm) will suffer in proportion to the intervention probability.
Despite the dependency on $P_{\GG}(\pi^*)$, we argue that \alg provides a resonable trade-off for safe RL thanks to its training safety and numerical stability. Moreover, we will discuss how to use data to improve the intervention rule $\GG$ to reduce this performance bias.

In the following, we first discuss the design of our advantage-based intervention rules (\cref{sec:intervention}) 
and provide details of the new MDP $\widetilde{\MM}$ (\cref{sec:absorbing mdp}). 
Then we state and prove the main result \cref{th:main theorem} (\cref{sec:main result}).
The omitted proofs for the results in this section can be found in \cref{app:missing proofs}.

\vspace{-1mm}
\subsection{Advantage-Based Intervention} \label{sec:intervention}
\vspace{-1mm}

We propose a family of intervention rules based on \emph{advantage functions}. Each intervention rule $\GG$ here is specified by a 3-tuple $(\overline{Q},\mu,\eta)$, where $\overline{Q}:\safeset\times\AA\to[0,1]$ is a state-action value estimator,  $\eta\in[0,1]$ is a threshold, and $\mu\in\Pi$ is a backup policy.
Given an arbitrary policy $\pi$, $\GG=(\overline{Q},\mu,\eta)$ constructs a new \emph{shielded} policy $\pi'$ 
based on an intervention set $\II$ defined by the advantage-like function $\overline{A}(s,a) \coloneqq \overline{Q}(s,a) - \overline{Q}(s,\mu) $:
\begin{align} \label{eq:intervention set}
    \II \coloneqq \{ (s,a) \in \safeset \times \AA \;:\; \overline{A}(s,a) > \eta \}.
\end{align}
When sampling $a$ from $\pi'(\cdot|s)$ at some $s\in\safeset$, it first samples $a_-$ from $\pi(\cdot|s)$.
If $(s,a_-) \notin \II$, it executes $a=a_-$.
Otherwise, it samples $a$ according to  $\mu(\cdot|s)$.
Mathematically, $\pi'$ is described by the conditional distribution
\begin{align} \label{eq:shielded policy}
    \pi'(a|s) \coloneqq \pi(a|s) \one\{(s,a)\notin\II\} + \mu(a|s) w(s),
\end{align}
where $w(s) \coloneqq 1 - \sum_{\tilde a : (s,\tilde a) \in \II} \pi(\tilde a|s)$.
Note that $\pi'$ may still take actions in $\II$ when $\mu$ has non-zero probability assigned to those actions.

By running the shielded policy contructed by the advantage function $\overline{A}$, SAILR controls the safety relative to the backup policy $\mu$ with respect to $d_0$. 
As we will show later, if the relative safety for each time step (i.e., advantage) is close to zero, then the relative safety overall is also close to zero (i.e. $\overline{V}^{\pi'}(d_0)\leq \delta$).
Note that the sheilded policy $\pi'$, while satisfying $\overline{V}^{\pi'}(d_0)\leq \delta$, can generally visit (with low probability) the states where $\overline{V}^\mu(s)>0$ (e.g., $=1$). At these places where $\mu$ is useless for safety, we need an intervention rule that naturally deactivates and lets the learner explore. Our advantage-based rule does exactly that. On the contrary, designing an intervention rule direclty based on Q-based functions, as in \citep{bharadhwaj2021conservative,thananjeyan2020recovery,eysenbach2018leave,srinivasan2020learning}, can be overly conservative in this scenario.

\vspace{-1mm}
\subsubsection{Motivating Example}
\vspace{-1mm}

Let us use an example to explain why the advantage-based rule works. Suppose we have a baseline policy $\mu$ that is safe starting at the intial state of the MDP $\MM$ (i.e., $\overline{V}^\mu(d_0)$ is small).
We can use $\mu$ as the backup policy and construct an intervention rule $\GG = (\overline{Q}^\mu, \mu, 0)$, where we recall $\overline{Q}^\mu$ denotes the state-action value of $\mu$ for the cost-based MDP $\overline{\MM}$.
Because the intervention set in \eqref{eq:intervention set} only allows actions that are no more unsafe than than backup policy $\mu$ in execution, intuitively we see that the intervenend policy $\pi'$ will be at least as safe as the baseline policy $\mu$.
Indeed, we can quickly verify this by the performance difference lemma~(\cref{lm:pdl}):
${\overline{V}^{\pi'}(d_0) = \overline{V}^{\mu}(d_0) + \frac{1}{1-\gamma}\E_{d^{\pi'}}[\overline{A}^\mu(s,a)] \leq \overline{V}^{\mu}(d_0)}$.
Importantly, in this example, we see that the safety of $\pi'$ is ensured without requiring $\overline{V}^{\mu}(s)$ to be small for any $s\in\SS$, but only starting from states sampled from  $d_0$.

\vspace{-1mm}
\subsubsection{General Rules} \label{sec:gr}
\vspace{-1mm}

We now generalize the above motivating example to a class of \emph{admissible} intervention rules.
\begin{definition}[$\sigma$-Admissible Intervention Rule] \label{def:admissible intervention}
We say an intervention rule $\GG=(\overline{Q},\mu,\eta)$ is \emph{$\sigma$-admissible} if, for some $\sigma \geq 0$, the following holds for all $s\in\safeset$ and $a\in\AA$:
    \begin{align} 
        & \overline{Q}(s,a) \in [0,\gamma] \label{eq:admissible intervention (upper bound)} \\
        & \overline{Q}(s,a) + \sigma \geq c(s,a) + \gamma \E_{s'|s,a}[ \overline{Q}(s',\mu) ]  \label{eq:admissible intervention (conservative)},
    \end{align}
    where we recall $c(s,a) = \one\{ s=\violation \}$.
    If the above holds with $\sigma=0$, we say $\GG$ is \emph{admissible}.
\end{definition}

One can verify that the previous example $\GG = (\overline{Q}^\mu, \mu, 0)$ is admissible.
But more generally, an admissible intervention rule with a backup policy $\mu$ can use $\overline{Q}\neq\overline{Q}^\mu$.
In a sense, admissibility (with $\sigma = 0$) only needs $\overline{Q}$ to be a conservative version of $\overline{Q}^\mu$, because $\overline{Q}^\mu(s,a) = c(s,a) + \gamma \E_{s' |s,a}[ \overline{Q}^\mu(s,\mu)]$ and \eqref{eq:admissible intervention (conservative)} uses an upper bound; the $\sigma$ term is a slack to allow for non-conservative $\overline{Q}$. 
More precisely, we have the following relationship. 
\begin{restatable}{proposition}{PessimsticEstimate} \label{th:pessimistic estimate}
If $\GG=(\overline{Q},\mu,\eta)$ is $\sigma$-admissible,
then $\overline{Q}^\mu(s,a) \leq \overline{Q}(s,a) + \frac{\sigma}{1-\gamma}$ for all $s \in \safeset$ and $a \in \AA$.
\end{restatable}
The condition in \eqref{eq:admissible intervention (conservative)} is also closely related to the concept and theory of improvable heuristics in \citep{cheng2021heuristicguided} (i.e., we can view the $\overline{Q}(s,\mu)$ as a heurisitic for safety), where the authors show such $\overline{Q}$ can be constructed by pessimistic offline RL methods.

\vspace{-1mm}
\paragraph{Examples}
We discuss several ways to construct admissible intervention rules.
From \cref{def:admissible intervention}, it is clear that if $\GG=(\overline{Q},\mu,\eta)$ is $\sigma$-admissible,  then $\GG$ is also $\sigma'$-admissible for any $\sigma'\geq\sigma$ (in particular, $(\overline{Q},\mu,\eta)$ is $\gamma$-admissible if $\overline{Q}(s,a) \in [0,\gamma]$).
So we only discuss the minimal $\sigma$.

\begin{restatable}[Intervention Rules]{proposition}{ExampleInterventionRule} \label{th:example admissible intervention rules}
    The following are true.
    \begin{enumerate}[leftmargin=.3cm,labelsep=0cm]
        \vspace{-1.5mm}
        \item \textbf{Baseline policy}:
            Given a baseline policy $\mu$ of $\MM$, $\GG = (\overline{Q}^\mu,\mu, \eta)$ or $\GG = (\overline{Q}^\mu,\mu^+, \eta)$ is admissible, where $\mu^+$ is the greedy policy that treats $\overline{Q}^\mu$ as a cost.
        
        \item \textbf{Composite intervention}:
            Given $K$ intervention rules $\{\GG_k\}_{k=1}^K$, where each $\GG_k = (\overline{Q}_k, \mu_k, \eta)$ is $\sigma_k$-admissible.
            Define $\overline{Q}_{\min}(s,a) = \min_k \overline{Q}_k(s,a)$ and let $\mu_{\min}$ be the greedy policy w.r.t. $\overline{Q}_{\min}$, and $\sigma_{\max} = \max_k \sigma_k$.
            Then, $\GG = (\overline{Q}_{\min}, \mu_{\min}, \eta)$ is $\sigma_{\max}$-admissible.
        
        \item \textbf{Value iteration}:
            Define $\overline\TT$ as $\overline\TT Q(s, a) \coloneqq c(s,a) + \gamma \E_{s' \sim P | s,a} [\min_{a'} Q(s', a')]$.
            If $\GG=(\overline{Q},\mu,\eta)$ is $\sigma$-admissible, then $\GG^k = (\overline{\TT}^k \overline{Q}, \mu^k, \eta)$ is $\gamma^k \sigma$-admissible, where $ \mu^k$ is the greedy policy that treats $\overline{\TT}^k \overline{Q}$ as a cost.
        
        \item \textbf{Optimal intervention}:
            Let $\overline\pi^*$ be an optimal policy for $\overline\MM$, and let $\overline{Q}^*$ be the corresponding state-action value function.
        Then $\GG^*=(\overline{Q}^*, \overline{\pi}^*, \eta)$ is admissible.
        
        \item \textbf{Approximation}:
            For $\sigma$-admissible $\GG = (\overline{Q}, \mu, \eta)$, consider $\hat{Q}$ such that $\hat{Q}(s,a) \in [0,\gamma]$ for all $s\in\safeset$ and $a\in\AA$.
            If $\|\hat{Q}- \overline{Q}\|_\infty \leq \delta$,
            then $\hat{\GG} = (\hat{Q}, \mu, \eta)$ is $(\sigma+(1+\gamma)\delta)$-admissible. 
    \end{enumerate}
\end{restatable}

\cref{th:example admissible intervention rules} provides recipes for constructing $\sigma$-admissible intervention rules for safe RL, such as leveraging existing baseline policies in a system (Examples 1 and 2) and performing short-horizon planning (Example 3; namely model-predictive control \citep{bertsekas2017dynamic}).
Moreover, \cref{th:example admissible intervention rules} hints that we can treat designing intervention rules as finding the optimal state-action value function $\overline{Q}^*$ in the cost-based MDP $\overline{\MM}$ (Example 4). 
Later in \cref{sec:analysis of intervention rules}, we prove that this intuition is indeed correct: among all intervention rules that provide optimal safety, the rule $\GG^*=(\overline{Q}^*, \overline{\pi}^*, 0)$ provides the largest free space for data collection (i.e., small $P_\GG(\pi^*)$ in \cref{th:informal theorem}) among the safest intervention rules.
Finally, \cref{th:example admissible intervention rules} shows that an approximation of any $\sigma$-admissible intervention rule (such as one learned from data or inferred from an inaccurate model, see \citep{cheng2021heuristicguided}) is also a reasonable intervention rule (Example 5). 
As learning continues in \alg, we can use the newly collected data from $\MM$ to refine our estimate of the ideal $\overline{Q}$, such as by performing additional policy evaluation for $\mu$ or policy optimization to find $\overline{Q}^*$ of the cost-based MDP $\overline{\MM}$.

\vspace{-1mm}
\paragraph{General Backup Policies}
To conclude this section, we briefly discuss how to extend the above results to work with general backup policies that may take actions outside $\AA$ (i.e., the actions the learner policy aims to use), as in \cite{turchetta2020safe}. For example, such a backup policy can be implemented through an external kill switch in a robotics system.
For \alg's theoretical guarantees to hold in this case, we require one extra assumption: 
for all $(s,a)\in\II$ that can be reached from $d_0$ with some policy, there must be some $a'\in\AA$ such that $\overline{A}(s,a') = \overline{Q}(s,a')-\overline{Q}(s,\mu) \leq \eta$. 
In other words, for every state-action we can reach from $d_0$ that will be overridden, there must an alternative action \emph{in the agent's action space} $\AA$ that keeps the agent's policy from being intervened.
This condition is a generalization of \cref{as:viable action} introduced later for our analysis (a condition we call \emph{partial}), which is essential to the unconstrained policy optimization reduction in \alg (\cref{sec:analysis of absorbing MDPs}).
Note that while this condition holds trivially when backup policy $\mu$ takes only actions in $\AA$, generally the validity of this condition depends on the details of $\mu$ and transition dynamics $P$.

\vspace{-1mm}
\subsection{Absorbing MDP} \label{sec:absorbing mdp}
\vspace{-1mm}

\alg performs policy optimization by running a base RL algorithm $\FF$ to solve a new unconstrained MDP $\widetilde{\MM}$. In this section, we define $\widetilde{\MM}$ and discuss how to simulate experiences of $\pi$ in $\widetilde{\MM}$ by running the shielded policy $\pi'= \GG(\pi)$ in the original MDP $\MM$.

Given the MDP $\MM=(\SS,\AA,P,r,\gamma)$ and the intervention set $\II$ in~\eqref{eq:intervention set}, we define $\widetilde{\MM}=(\widetilde{\SS},\AA,\widetilde{P},\widetilde{r},\gamma)$ as follows:
Let $\intervened$ denote an absorbing state and $\widetilde{R}\leq0$ be some problem-independent constant. The new MDP $\widetilde{\MM}$ has the state space $\widetilde{\SS} = \SS\cup\{\intervened\}$ and modified dynamics and reward,
\begin{align}
    \widetilde{r}(s,a) &=
    \begin{cases}
        \widetilde R, & (s, a) \in \II \\
        0, & s = \intervened \\
        r(s,a), & \text{otherwise}
    \end{cases} \label{eq:modified reward} \\
    \widetilde{\PP}(s'|s,a) &=
    \begin{cases}
        \one\{s' = \intervened\}, & (s,a) \in \II ~\text{or}~ s = \intervened \\
        \PP(s'|s,a), &\text{otherwise.}
    \end{cases} \label{eq:modified dynamics}
\end{align}
Since $\intervened$ is absorbing, given a policy $\pi$ defined on $\MM$, without loss of generality we extend its definition on $\widetilde{\MM}$ by setting $\pi(a | \intervened)$ to be the uniform distribution over $\AA$.
A simple example of this construction is shown in~\cref{fig:toy example}.

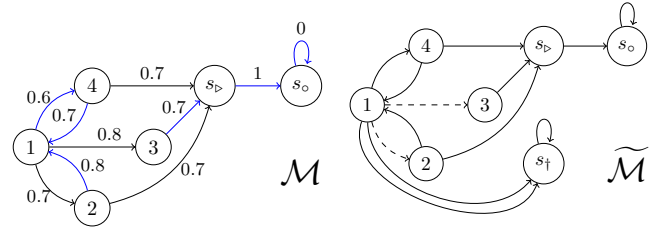
\begin{figure}[t]
    \centering
    \begin{subfigure}{0.23\textwidth}
        \centering
        \resizebox{1.1\textwidth}{!}{%
        \begin{tikzpicture}[->, node distance=1.5cm]
    \node[draw, circle] (1) {$1$};
    \node[draw, circle, below right of=1] (2) {$2$};
    \node[draw, circle, above right of=2] (3) {$3$};
    \node[draw, circle, above right of=1] (4) {$4$};
    \node[draw, circle, above right of=3] (violation) {\violation};
    \node[draw, circle, right of=violation] (absorbing) {\absorbing};
    \node[below of=absorbing] (M) {{\LARGE $\MM$}};

    \path[every path/.style={blue}]
    (1) edge[bend left] node [left, pos=0.7] {{\footnotesize {\color{black} $0.6$}}} (4)
    (2) edge [bend right] node [right] {{\footnotesize {\color{black} $0.8$}}} (1)
    (3) edge node [left, pos=0.8] {{\footnotesize {\color{black} $0.7$}}} (violation)
    (4) edge [bend left] node [left, pos=0.2] {{\footnotesize {\color{black} $0.7$}}} (1)
    (violation) edge node [above] {{\footnotesize {\color{black} $1$}}} (absorbing)
    (absorbing) edge [loop above] node {{\footnotesize {\color{black} $0$}}} (absorbing);

    \path
    (1) edge node [above, pos=0.7] {{\footnotesize $0.8$}} (3)
    (1) edge[bend right] node [left, pos=0.7] {{\footnotesize $0.7$}} (2)
    (2) edge [bend right] node [right] {{\footnotesize $0.7$}} (violation)
    (4) edge node [above] {{\footnotesize $0.7$}} (violation);
\end{tikzpicture}
        }
    \end{subfigure}
    \hfill
    \begin{subfigure}{0.23\textwidth}
        \centering
        \resizebox{1.1\textwidth}{!}{%
        \begin{tikzpicture}[->, node distance=1.5cm]
    \node[draw, circle] (1) {$1$};
    \node[draw, circle, below right of=1] (2) {$2$};
    \node[draw, circle, above right of=2] (3) {$3$};
    \node[draw, circle, above right of=1] (4) {$4$};
    \node[draw, circle, above right of=3] (violation) {\violation};
    \node[draw, circle, right of=violation] (absorbing) {\absorbing};
    \node[draw, circle, below right of=3] (intervened) {\intervened};
    \node[right of=intervened] (M) {{\LARGE $\widetilde\MM$}};

    \path[every path/.style={dashed}]
        (1) edge[bend right] node [right] {} (2)
        (1) edge node [right] {} (3);

    \path
        (1) edge[bend right=70] node [right] {} (intervened)
        (1) edge[bend right=90] node [right] {} (intervened);

    \path
        (1) edge[bend left] node [right] {} (4)
        (2) edge [bend right] node [left] {} (1)
        (3) edge node [right] {} (violation)
        (4) edge [bend left] node [left] {} (1)
        (violation) edge node [right] {} (absorbing)
        (absorbing) edge [loop above] node {} (absorbing)
        (intervened) edge [loop above] node {} (intervened)
        (2) edge [bend right] node [right] {} (violation)
        (4) edge node [right] {} (violation);
\end{tikzpicture}
        }
    \end{subfigure}
    \caption{A simple example of the construction of $\widetilde\MM$ from $\MM$ using advantage-based intervention given by some $\GG = (\overline{Q}, \mu, \eta)$. In $\MM$, the transitions are deterministic, and the blue arrows correspond to actions given by $\mu$. The edge weights correspond to $\overline Q$, and $\GG$ can be verified to be $0.25$-admissible when $\gamma = 0.9$. The surrogate MDP $\widetilde\MM$ is formed upon intervention with $\eta = 0.05$. The transitions $1 \rightarrow 2$ and $1 \rightarrow 3$ are replaced with transitions to the absorbing state $\intervened$.}
    \label{fig:toy example}
\vspace{-2mm}
\end{figure}

Compared with the original $\MM$, the new MDP $\widetilde{\MM}$ has more absorbing state-action pairs and assigns lower rewards to them.
When the agent takes some $(s,a)\in\II$ in $\widetilde{\MM}$, it goes to an absorbing state $\intervened$ and receives a \emph{non-positive} reward.
Thus, the new MDP $\widetilde\MM$ gives larger penalties for taking intervened state-actions than for going into $\unsafeset$, where we only receive zero reward.
This design ensures that any nearly-optimal policy of $\widetilde\MM$ will (when run in $\MM$) have high reward and low probability of visiting intervened state-actions.
As we will see, as long as $\GG$ provides \emph{safe} shielded polices, solving $\widetilde{\MM}$ will lead to a safe policy with potentially good performance in the original MDP $\MM$ even after we lift the intervention.

To simulate experiences of a policy $\pi$ in $\widetilde\MM$, we simply run $\pi' = \GG(\pi)$ in the original MDP $\MM$ and collect samples until the intervention triggers (if at all).
Specifically, suppose running $\pi'$ in $\MM$ generates a trajectory $\xi = (s_0, a_0, \dots, s_T, a_T', \dots)$, where $T$ is the time step of intervention and $a_T'$ is the first action given by the backup policy $\mu$.
Let $a_T$ be the corresponding action from $\pi$ that was overridden.
We construct the trajectory $\widetilde\xi$ that would be generated by running ${\color{blue} \pi}$ in ${\color{blue} \widetilde\MM}$ by setting $\widetilde\xi = (s_0, a_0, \dots, s_T, a_T, \widetilde{s}_{T+1}, \widetilde{a}_{T+1}, \dots)$, where $\widetilde{s}_\tau = \intervened$ and $\widetilde{a}_\tau$ is arbitrary for any $\tau \geq t+1$.
This is valid since the two MDPs $\MM$ and $\widetilde\MM$ share the same dynamics until the intervention happens at time step $T$.

\vspace{-1mm}
\subsection{Theoretical Analysis} \label{sec:main result}
\vspace{-1mm}

We state the main theoretical result of \alg, which includes the informal \cref{th:informal theorem} as a special case.
\begin{restatable}[Performance and Safety Guarantee at Deployment]{theorem}{SuboptimalityBound} \label{th:main theorem}
    Let $\widetilde{R} = -1$ and $\GG$ be $\sigma$-admissible.
    If $\hat\pi$ is an $\varepsilon$-suboptimal policy for $\widetilde{\MM}$, then, for any comparator policy $\pi^*$, the following performance and safety guarantees hold for $\hat\pi$ in $\MM$:
\begin{align*}
    V^{\pi^*}(d_0) - V^{\hat\pi}(d_0)
    &\leq \frac{2}{1-\gamma} P_\GG(\pi^*) + \varepsilon \\
    \overline{V}^{\hat\pi}(d_0)
    &\leq \overline{Q}(d_0, \mu) + \frac{\min\{\sigma + \eta, 2\gamma\}}{1-\gamma} + \varepsilon,
\end{align*}
    where $P_{\GG}(\pi^*) \coloneqq (1-\gamma) \sum_{h=0}^\infty \gamma^h  \prob(\xi^h \cap \II \neq \varnothing \mid \pi^*, \MM)$ is the probability that $\pi^*$ visits $\II$ in $\MM$.
\end{restatable}

\cref{th:main theorem} shows that, when the base RL algorithm $\FF$ finds an $\varepsilon$-suboptimal policy $\hat{\pi}$ in $\widetilde{\MM}$, this policy $\hat{\pi}$ is also close to $\varepsilon$-suboptimal in the CMDP in~\eqref{eq:CMDP formulation}, as long as running the comparator policy $\pi^*$ in $\MM$ will result in low probability of visiting state-actions that would be intervened by $\GG$  (i.e., $P_\GG(\pi^*)$ is small).
In addition, the policy $\hat{\pi}$ is almost as safe as the backup policy $\mu$, since $\overline{Q}(d_0,\mu)$ can be viewed as an upper bound of $\overline{Q}^\mu(d_0,\mu)$.
The safety deterioration can be made small when the suboptimality $\varepsilon$, intervention threshold $\eta$, and imperfect admissibility $\sigma$ of $\GG$ are small.
The proof of \cref{th:main theorem} follows directly from \cref{th:shielded policy is safe} and \cref{th:performance and safety} below, which are main properties of the advantage-based intervention rules and the absorbing MDPs in \alg.
We now discuss these properties in more detail.

\vspace{-1mm}
\subsubsection{Intervention Rules} \label{sec:analysis of intervention rules}
\vspace{-1mm}

First, we show that  the shielded policy $\pi'$ produced by a $\sigma$-admissible intervention rule $\GG=(\overline{Q},\mu,\eta)$ has a small unsafe cost if backup policy $\mu$ has a small cost.
\begin{restatable}[Safety of Shielded Policy]{theorem}{IntervenedPolicyIsSafe}
    \label{th:shielded policy is safe}
    Let $\GG=(\overline{Q}, \mu, \eta)$ be $\sigma$-admissible as per \cref{def:admissible intervention}.
    For any policy $\pi$, let $\pi' = \GG(\pi)$. Then,
    \begin{align} \label{eq:max unsafety}
        \overline{V}^{\pi'}(d_0) \leq  \overline{Q}(d_0,\mu) + \frac{\min\{\sigma+\eta,2\gamma\} }{1-\gamma} .
    \end{align}
\end{restatable}

Next we provide a formal statement that
$\GG^*=(\overline{Q}^*,\overline{\pi}^*,0)$ is the optimal intervention rule that gives the largest free space for policy optimization, among the safest intervention rules. The size of the free space provided $\GG^*$ is captured as $\supp_{\SS\times\AA}(\widetilde{d}^{*,\pi})$, which can be interpreted as the state-actions that $\GG^*(\pi) $ can explore before any intervention is triggered.

\begin{restatable}{proposition}{OptimalIntervention} \label{th:optimal intervention}
    Let $\overline\pi^*$ be an optimal policy for $\overline\MM$, $\overline{Q}^*$ be its state-action value function, and $\overline{V}^*$ be its state value function.
    Let $G_0 = \{ (\overline{Q}, \mu, 0) : (\overline{Q}, \mu, 0)~\text{is admissible}, \; \overline{Q}(d_0, \mu) = \overline{V}^*(d_0)\}$.
    Let $\GG^* = (\overline{Q}^*, \overline\pi^*, 0) \in G_0$.
    Consider arbitrary $\GG \in G_0$ and policy $\pi$.
    Let $\widetilde{\MM}$ and $\widetilde{\MM}^*$ be the absorbing MDPs induced by $\GG$ and $\GG^*$, respectively, and let $\widetilde{d}^\pi$ and $\widetilde{d}^{*,\pi}$ be their state-action distributions of $\pi$. Then,
    \begin{align*}
      \supp_{\SS\times\AA}(\widetilde{d}^{\pi}) \subseteq \supp_{\SS\times\AA}(\widetilde{d}^{*,\pi}),
    \end{align*}
    where $\supp_{\SS\times\AA}(d)$ denotes the support of a distribution $d$ when restricted on $\SS\times\AA$.
\end{restatable}

Finally, we highlight a property of the intervention set $\II$ of our advantage-based rules, which is crucial for the unconstrained MDP reduction described in the next section.
\begin{definition} \label{as:viable action}
    A set $\XX \subset \safeset \times \AA$ is called \emph{partial} if for every $(s, a) \in \XX$, there is some $a' \in \AA$ such that ${(s, a') \notin \XX}$.
\end{definition}
\begin{proposition} \label{th:partial}
If $\eta \geq 0$, then $\II$ in \eqref{eq:intervention set} is partial.
\end{proposition}
\begin{proof}
    For $(s, a) \in \II$,  define $a' = \argmin_{a'' \in \AA} \overline{Q}(s, a'')$.
    Because
    $
        \overline{A}(s, a') = \overline{Q}(s, a') - \overline{Q}(s, \mu) \leq 0 \leq \eta,
    $
    we conclude that $(s, a') \notin \II$.
\end{proof}

\vspace{-2mm}
\subsubsection{Abosrbing MDP}  \label{sec:analysis of absorbing MDPs}
\vspace{-1mm}

As discussed in \cref{sec:absorbing mdp}, the new MDP $\widetilde{\MM}$ provides a pessimistic value estimate of $\MM$ by penalizing trajectories that trigger the intervention rule $\GG$. 
Precisely, we can show that the amount of pessimism introduced on a policy $\pi$ is proportional to $P_{\GG}(\pi)$ (the probability of triggering the intervention rule $\GG$ when running $\pi$ in $\MM$).
\begin{restatable}{lemma}{ValueOffset} \label{th:value offset}
    For every policy $\pi$, it holds that
     \begin{align*}
         |\widetilde{R}| \; P_\GG(\pi) \leq V^\pi(d_0) - \widetilde{V}^\pi(d_0) \leq \left( |\widetilde{R}| + \frac{1}{1-\gamma} \right) P_\GG(\pi) .
     \end{align*}
\end{restatable}
\vspace{-2mm}

As a result, one would intuitively imagine that an optimal policy of $\widetilde{\MM}$ would never visit the intervention set $\II$ at all. Below we show that this intuition is correct. 
Importantly, we highlight that this property holds \emph{only} because the intervention set $\II$ used here is \emph{partial}~(\cref{th:partial}).
If we were to construct an absorbing MDP $\widetilde{\MM}'$ described in \cref{sec:absorbing mdp} using an arbitrary non-partial subset $\II' \subseteq \safeset\times\AA$, then the optimal policy of  $\widetilde{\MM}'$ can still enter $\II'$ for any $\widetilde{R}>-\infty$, because an optimal policy of $\widetilde{\MM}'$ can use earlier rewards to mitigate penalties incurred in $\II'$~(\cref{sec:non-partial}).
\begin{restatable}{proposition}{OptimalPolicyIsNotIntervened} \label{th:optimal policy is not intervened}
    If $\widetilde R$ is negative and $\GG$ induces a partial $\II$, then every optimal policy $\widetilde{\pi}^*$ of $\widetilde{\MM}$ satisfies $P_{\GG}(\widetilde{\pi}^*) = 0$.
\end{restatable}

The partial property of $\II$ enables our unconstrained MDP reduction, which relates the performance and safety of a policy $\pi$ in the original MDP $\MM$ to the suboptimality in the new MDP $\widetilde{\MM}$ and the safety of $\pi'=\GG(\pi)$.
\begin{restatable}[Suboptimality in $\widetilde\MM$ to Suboptimality and Safety in $\MM$]{proposition}{PerformanceAndSafety}\label{th:performance and safety}
    Let $\widetilde R$ be negative.
    For some policy $\pi$, let $\pi'$ be the shielded policy defined in \eqref{eq:shielded policy}.
    Suppose $\pi$ is $\varepsilon$-suboptimal for $\widetilde\MM$.
    Then, for any comparator policy $\pi^*$, the following performance and safety guarantees hold for $\pi$ in $\MM$:
    \begin{align*}
        V^{\pi^*}(d_0)  -  V^\pi(d_0)
        &\leq \left(  |\widetilde{R}| + \frac{1}{1-\gamma}  \right) P_\GG(\pi^*) + \varepsilon \\
        \overline{V}^\pi(d_0)
        &\leq \overline{V}^{\pi'}(d_0) + \frac{\varepsilon}{|\widetilde{R}|} .
    \end{align*}
\end{restatable}

\vspace{-4mm}
\section{Related Work}
\vspace{-1mm}

CMDPs~\citep{altman1999constrained} have been a popular framework for safe RL as it side-steps the reward design problem for ensuring safety in a standard MDP~\citep{geibel2005risk,shalev2016safe}. 
Most existing CMDP-based safe RL algorithms closely follow algorithms in the constrained optimization literature~\citep{bertsekas2014constrained}. 
They can be classified into either online or offline schemes. 
Online schemes learn by coupling the iteration of a numerical optimization algorithm (notably primal-dual gradient updates) with data collection~\citep{borkar2005actor,chow2017risk,tessler2018reward,bohez2019value}, and these algorithms have also been studied in the exploration context~\citep{ding2020provably,qiu2020upper,efroni2020exploration}.
However, they have no guarantees on policy safety during training.
Offline schemes \citep{achiam2017constrained,bharadhwaj2021conservative,le2019batch,efroni2020exploration}, on the other hand, separate optimization and data collection. They conservatively enforce safety constraints on every policy iterate but are more difficult to scale up.
Many of these constrained algorithms for CMDPs, however, have worse numerical stability compared with typical RL algorithms for MDPs, because of the nonconvex saddle-point of the CMDP~\citep{lee2017first,chow2018lyapunov}.

Another line of safe RL research uses control-theoretic techniques to enforce safe exploration, though only few provide guarantees with respect to the CMDP in \eqref{eq:CMDP formulation}.
These methods include restricting the agent to take actions that lead to next-state safety~\citep{dalal2018safe, wabersich2018safe} or states where a safe backup exists~\citep{hans2008safe,polo2011safe,li2020robust}.
Other works consider more structured shielding approaches, including those with temporal logic safety rules and backup policies~\citep{alshiekh2018safe} and neurosymbolic policies~\citep{anderson2020neurosymbolic} whose safety can be checked easily.
Many of these approaches require strong assumptions on the MDP (e.g., taking an action to ensure the next state's safety being sufficient to imply all future states will continue to have such safe actions available).
Algorithms based on Lyapunov functions and reachability~\citep{perkins2002lyapunov,chow2018lyapunov,chow2019lyapunov,berkenkamp2017safe,fisac2018general} address the long-term feasibility issue, but they are more complicated than common RL algorithms.
We note that our admissible intervention rules in \eqref{eq:admissible intervention (conservative)} can be viewed as a state-action Lyapunov function.

To the best of our knowledge, \alg is the first unconstrained method that provides formal guarantees with respect to the CMDP objective.
The closest work to ours is~\citep{turchetta2020safe}, which also uses the idea of intervention for training safety and trains the agent in a new MDP that discourages visiting intervened state-actions. However, their algorithm, \textsc{CISR}, is still based on calling CMDP subroutines~\citep{le2019batch}. They neither specify how the intervention rules can be constructed nor provide performance guarantees.
By comparsion, we provide a general recipe of intervention rules and obtain the properties desired in \citep{turchetta2020safe} by simply unconstrained RL.

\begin{figure*}[!ht]
  \centering
  {\small {\bf Episode return without intervention \quad Episode length without intervention \qquad Safety violations during training}} \\
  \begin{subfigure}{0.27\textwidth}
    \includegraphics[width=\textwidth]{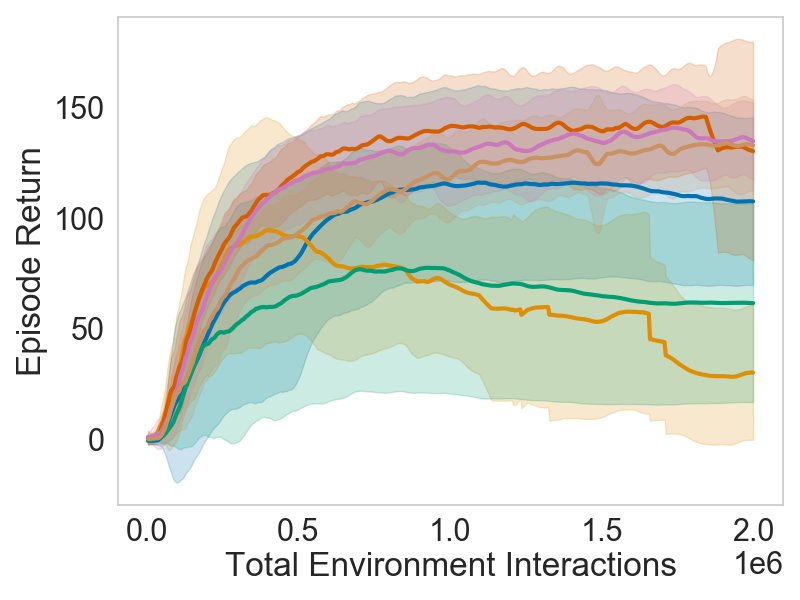}
  \end{subfigure}
  \begin{subfigure}{0.27\textwidth}
    \includegraphics[width=\textwidth]{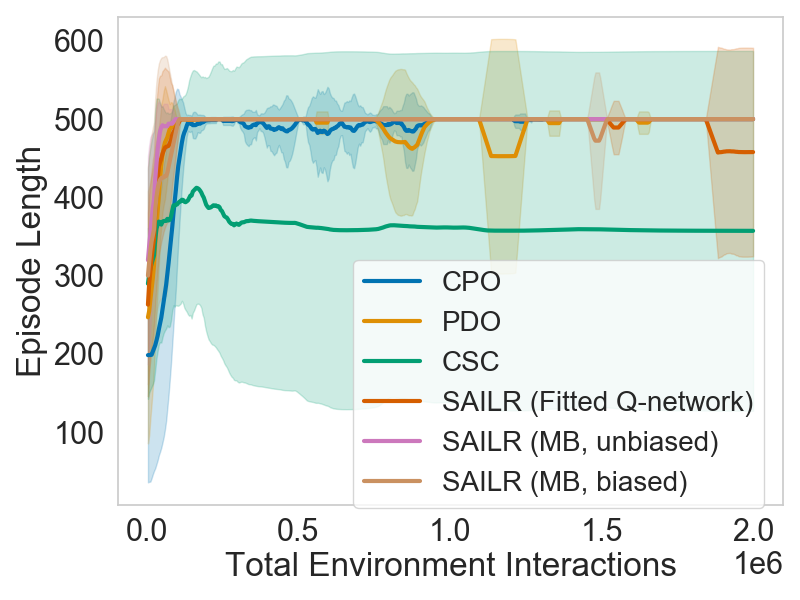}
    \caption{Results for point} 
    \label{fig:point}
  \end{subfigure}
  \begin{subfigure}{0.27\textwidth}
    \includegraphics[width=\textwidth]{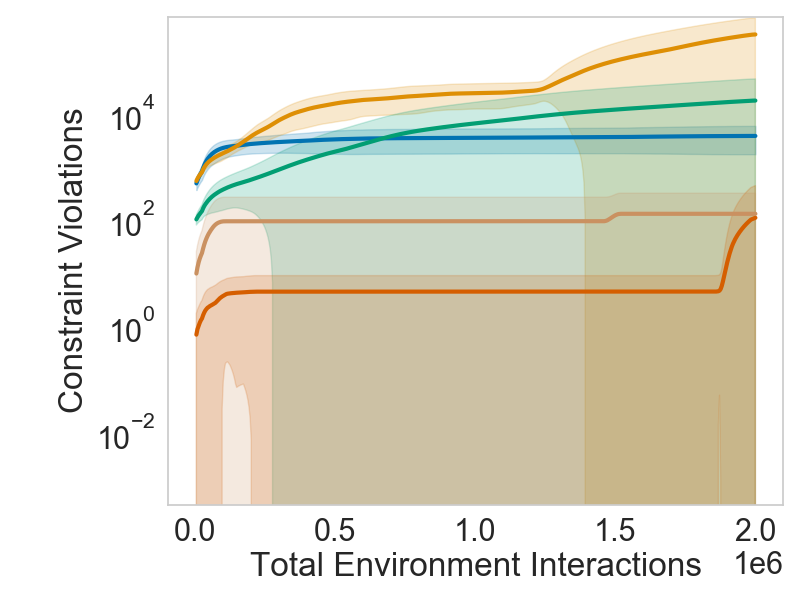}
  \end{subfigure} \\
  \begin{subfigure}{0.27\textwidth}
    \includegraphics[width=\textwidth]{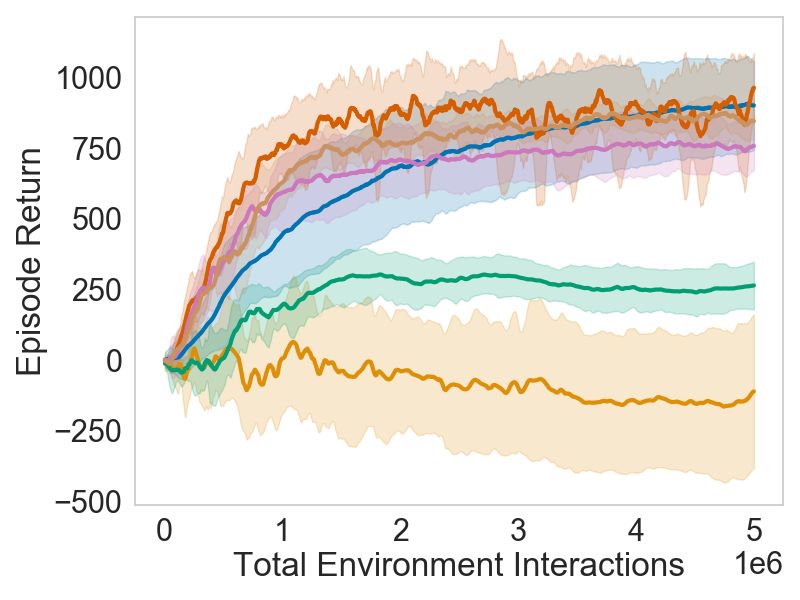}
  \end{subfigure}
  \begin{subfigure}{0.27\textwidth}
    \includegraphics[width=\textwidth]{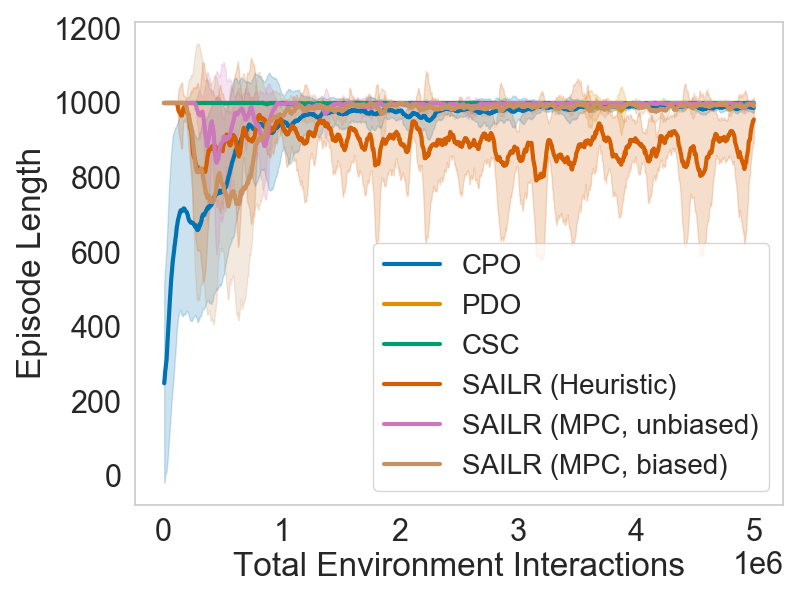}
    \caption{Results for half-cheetah}
    \label{fig:cheetah}
  \end{subfigure}
  \begin{subfigure}{0.27\textwidth}
    \includegraphics[width=\textwidth]{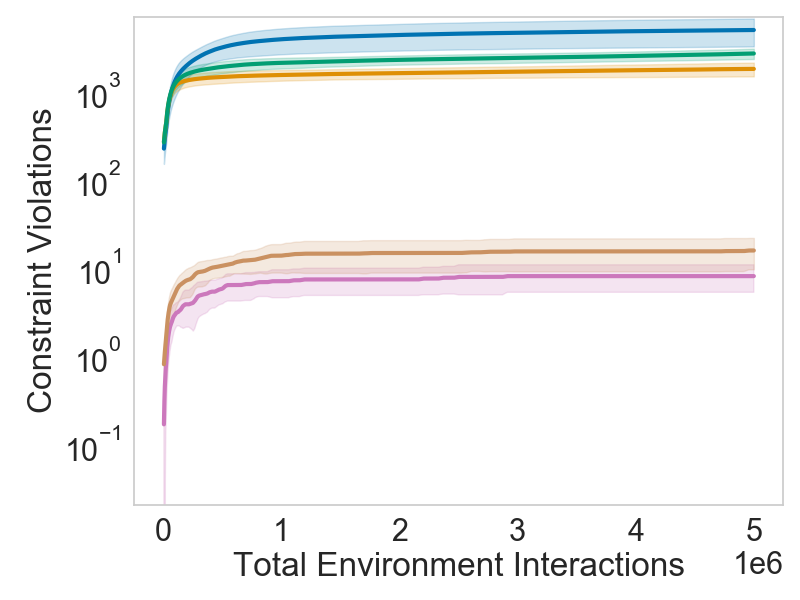}
  \end{subfigure}
  \caption{Results of \alg and baseline CMDP-based methods.
  Overall \alg dramatically reduces the amount of safety constraint violations while still having large returns at deployment.
Plots in a row share the same legend.
All error bars are $\pm 1$ standard deviation over 10 (point robot) or 8 (half-cheetah) random seeds. Any curve not plotted in the third column corresponds to zero safety violations.}
  \vspace{-5mm}
\end{figure*}

\vspace{-1mm}
\section{Experiments} \label{sec:exps}
\vspace{-1mm}

We conduct experiments to corroborate our theoretical analysis of \alg.
We aim to verify whether a properly designed intervention mechanism can drastically reduce the amount of unsafe trajectories generated in training while still resulting in good safety and performance in deployment.

Our experiments consider two different tasks:
\begin{enumerate*}[label=\textit{\arabic*)}]
\item A toy point robot based on~\citep{achiam2017constrained} that gets reward for following a circular path at high speed, but is constrained to stay in a region smaller than the target circle; and
\item a half-cheetah that gets reward equal to its forward velocity, with one of its links constrained to remain in a given height range, outside of which the robot is deemed to be unsafe.
\end{enumerate*}
In all experiments, when computing $\overline{Q}$, we opt to use a \emph{shaped} cost function in place of the original sparse indicator cost function to make our intervention mechanism more conservative (and hence the training process safer).
In particular, this shaped cost function is a function of the distance to the unsafe set and is an upper bound of the original sparse cost.
The appendix includes some additional experiments where the original sparse cost is used.

We implement \alg by using PPO~\citep{schulman2017proximal} as the RL subroutine. 
We also compare our approach to two CMDP-based approaches: CPO~\citep{achiam2017constrained} and a primal-dual optimization (PDO) algorithm~\citep{chow2017risk}.
For the PDO algorithm, we use PPO as the policy optimization subroutine and dual gradient ascent as the Lagrange multiplier update.
We also consider a variant of PDO, called CSC, where a learned conservative critic is used to filter unsafe actions~\citep{bharadhwaj2021conservative}.

\vspace{-2mm}
\subsection{Point Robot}
\vspace{-2mm}

Here \alg uses the intervention rule $\GG=(\mu, \overline{Q}, \eta)$: the baseline policy $\mu$ aims to stop the robot by deceleration.
The function $\overline{Q}$ is estimated by either querying a fitted Q-network or by rolling out $\mu$ on a dynamical model (denoted ``MB'' in \cref{fig:point}) of the point robot and querying a shaped cost function.
We consider both biased and unbiased models (details in \cref{app:point robot}).
\cref{fig:point} show the main experimental results, with all three instances of \alg outperforming the baselines on all three  metrics.
For \alg, the shielding prevents many safety violations, and the unconstrained approach allows for reliable convergence as opposed to the baselines which rely on elaborate constrained approaches.

\vspace{-3mm}
\subsection{Half-Cheetah}
\vspace{-2mm}

We consider two intervention rules in \alg:
a reset backup policy $\mu$ with a simple heuristic $\overline{Q}$ based on the predicted height of the link after taking a proposed action,
and a reset backup policy $\mu$ based on a sampling-based model predictive control (MPC) algorithm \citep{williams2017information,bhardwaj2021blending} with a model-based value estimate (i.e., $\overline{Q}\approx \overline{Q}^\mu$).
The simple heuristic uses a slightly smaller height range for intervention to attempt to construct a \emph{partial} intervention set (\cref{sec:main result}).
The MPC algorithm optimizes a control sequence over the same cost function.
The function $\overline{Q}$ is computed by rolling out this control sequence on the dynamical model and querying the cost function.
We also consider model bias in the MPC experiments (details in~\cref{app:cheetah}).

As with the point environment, \alg incurs orders of magnitude fewer safety violations than the baselines (right plot of \cref{fig:cheetah}), with all three instances having comparable deployment performance to that of CPO.
Though the heuristic intervention violates no constraints in training, it is consistently unsafe in deployment (middle plot), likely because the resulting intervention set is not partial.
On the other hand, MPC-based approaches are consistently safe in deployment, owing to its multi-step lookahead yielding an intervention rule that is likely to be $\sigma$-admissible (and therefore give an intervention set that is partial).

\vspace{-3mm}
\section{Conclusion}
\vspace{-2mm}

We presented an intervention-based method for safe reinforcement learning.
By utilizing advantage functions for intervention and penalizing an agent for taking intervened actions, we can use unconstrained RL algorithms in the safe learning domain.
Our analysis shows that using advantage functions for the intervention decision gives strong guarantees for safety during training and deployment, with the performance only limited by how often the true optimal policy would be intervened.
We also discussed ways of synthesizing good intervention rules, such as using value iteration techniques.
Finally, our experiments showed that the shielded policy violates few if any constraints during training while the corresponding deployed policy enjoys convergence to a large return.
\vspace{-3mm}

\vspace{-1mm}
\section*{Acknowledgements}
\vspace{-1mm}
This work was supported in part by ARL SARA CRA W911NF-20-2-0095.
We thank Mohak Bhardwaj for providing MPC code used in the half-cheetah experiment.
We thank Anqi Li for insightful discussions and Panagiotis Tsiotras for helpful comments on the paper.

\bibliography{refs}
\bibliographystyle{icml2021}

\clearpage
\onecolumn
\appendix

\section{Missing Proofs} \label{app:missing proofs}

\subsection{Useful Lemmas}

\begin{lemma} \label{lm:discounted value as average value}
    For any $\gamma$-discounted MDP with reward function $r$,
    the identity $V^\pi(d_0) = (1-\gamma) \sum_{h=0}^\infty \gamma^h U_h^\pi(d_0)$ holds, where
    $ U_h^\pi(d_0) = \E_{\rho^\pi}[  \sum_{t=0}^h  r(s_t, a_t) ] ]$ is the undiscounted $h$-step return.
\end{lemma}
\begin{proof}
    The proof follows from exchanging the order of summations:
    \begin{align*}
        (1-\gamma) \sum_{h=0}^\infty \gamma^h U_h^\pi(d_0)
        &= (1-\gamma) \sum_{h=0}^\infty \gamma^h \E_{\rho^\pi} \left[  \sum_{t=0}^h  r(s_t, a_t)  \right]\\
        &= (1-\gamma) \E_{\rho^\pi} \left[ \sum_{t=0}^\infty  r(s_t, a_t) \sum_{h=t}^\infty \gamma^h  \right]\\
        &= \E_{\rho^\pi} \left[ \sum_{t=0}^\infty  \gamma^t  r(s_t, a_t) \right] \\
        &= V^\pi(d_0)
    \end{align*}
\end{proof}

\begin{lemma}[Performance Difference Lemma~\citep{kakade2002approximately,cheng2020policy}] \label{lm:pdl}
    Let $\MM$ be an MDP and $\pi$ be a policy.
    For any function $f:\SS\to\R$ and any initial state distribution $d_0$, it holds that
    \begin{align*}
        V^\pi(d_0) - f(d_0) = \frac{1}{1-\gamma} \E_{(s, a) \sim d^\pi}[r(s,a) + \gamma \E_{s'|s,a}[f(s')] - f(s)]
    \end{align*}
\end{lemma}

\begin{corollary} \label{lm:simulation lemma}
    Let $\MM$ and $\widehat{\MM}$ be MDPs with common state and action spaces.
    For any policy $\pi$, the difference in value functions in $\MM$ and $\widehat{\MM}$ satisfies
    \[
        V^\pi(d_0) - \widehat{V}^\pi(d_0) = \frac{1}{1-\gamma} \E_{(s,a) \sim d^\pi}[(\DD^\pi \widehat{Q}^\pi)(s,a) ]
    \]
    where $\DD^\pi$ is the temporal-difference operator of $\MM$:
    \[
        (\DD^\pi Q)(s,a) \coloneqq (\BB^\pi Q)(s,a) - Q(s,a),
    \]
    and $\BB^\pi$ is the Bellman operator of $\MM$:
    \[
        (\BB^\pi Q)(s,a) \coloneqq r(s,a) + \gamma \E_{s' |s,a} [Q(s', \pi)].
    \]
\end{corollary}
\begin{proof}
    Set $f = \widehat{V}^\pi$ and observe that $\widehat{V}^\pi(s) = \widehat{Q}^\pi(s, \pi)$.
\end{proof}

\subsection{Proof of Equivalent CMDP Formulation in \cref{sec:problem formulation}} \label{app:proof of CMDP version}

Here we show that \eqref{eq:basic formulation} and \eqref{eq:CMDP formulation} are the same by proving  the equivalence
\begin{align} \label{eq:equivalence between constraints}
    (1-\gamma) \sum_{h=0}^\infty \gamma^h \prob(\xi_h\subset\safeset |\pi) \geq 1-\delta
     \quad  \Longleftrightarrow \quad
    \overline{V}^\pi(d_0) \leq \delta
\end{align}
By the definition of the cost function $c(s,a) = \one\{ s=\violation \}$ and absorbing property of $\unsafeset=\{\violation, \absorbing\}$, we can write
\begin{equation} \label{eq:probability of violation as sum of costs}
    1 - \prob(\xi_h\subset\safeset |\pi)
    = \prob( \violation \in \xi_h |\pi)
    = \E_{\rho^\pi} \left[ \sum_{t=0}^h c(s_t, a_t) \right]
\end{equation}
since $\violation$ can only appear at most once within $\xi_h$.
Substituting this equality into the negation of the chance constraint, 
\begin{align*}
    1 - (1-\gamma) \sum_{h=0}^\infty \gamma^h \prob(\xi_h\subset\safeset |\pi)
    &= (1-\gamma) \sum_{h=0}^\infty \gamma^h \E_{\rho^\pi} \left[ \sum_{t=0}^h c(s_t, a_t) \right]\\
    &= \E_{\rho^\pi} \left[ \sum_{t=0}^\infty \gamma^t c(s_t, a_t) \right] \\
    &= \overline{V}^\pi(d_0)
\end{align*}
where the second equality follows from \cref{lm:discounted value as average value}. Therefore, \eqref{eq:equivalence between constraints} holds.

\subsection{Proof for Intervention Rules in \cref{sec:method}}

\subsubsection{Admissible Rules and Pessimism}

\PessimsticEstimate*
\begin{proof}[Proof of \cref{th:pessimistic estimate}]
The proof follows by repeating the inequality of $\overline{Q}$.
    \begin{align*}
     \overline{Q}(s,a)  &
     \geq c(s,a) + \gamma \E_{s'|s,a}[\overline{Q}(s',\mu)] -\sigma \\
    &\geq c(s,a) + \gamma  \E_{s'|s,a} \left[ c(s',\mu),  + \gamma \E_{s''|s',\mu}[\overline{Q}^\mu(s'',\mu)]  \right] - (1+\gamma) \sigma   \\
    &\quad \vdots \\
    &\geq \overline{Q}^\mu(s',\mu) - \frac{\sigma}{1-\gamma} .
\end{align*}
\end{proof}

\subsubsection{Example Intervention Rules}
\ExampleInterventionRule*
\begin{proof}[Proof of \cref{th:example admissible intervention rules}]

    We show each intervention rule $\GG=(\overline{Q},\mu,\eta)$ below satisfies the admissibility condition
    \begin{align*}
        \overline{Q}(s,a) + \sigma \geq c(s,a) + \gamma \E_{s'\sim P |s,a}[ \overline{Q}(s',\mu) ] .
    \end{align*}
    For convenience, we define the Bellman operator $\overline\BB^\mu$ as $\overline\BB^\mu Q(s,a) \coloneqq c(s,a) + \gamma \E_{s' | s, a}[Q(s, \mu)]$.
    Then the admissibility condition can be written as $\overline{Q}(s,a) + \sigma \geq (\overline\BB^\mu \overline{Q})(s,a)$ for any $s \in \safeset$ and $a \in \AA$.
    Also, we write $\overline{Q} \in [0,\gamma]$ on $\safeset$ if $\overline{Q}(s,a)\in[0,\gamma]$ for all $s\in\safeset$ and $a\in\AA$.

    \begin{enumerate}
        \item \textbf{Baseline policy}:
            We know $\GG = (\overline{Q}^\mu, \mu, \eta)$ is admissible since $\overline{Q}^\mu  = \overline\BB^\mu \overline{Q}^\mu$.
            For $\GG = (\overline{Q}^\mu, \mu^+, \eta)$, we have $\overline{Q}^\mu \geq \overline\BB^{\mu^+} \overline{Q}^\mu$ since $\mu^+$ is greedy with respect to $\overline{Q}^\mu$.
            Also, by the definition of the cost $c$ and transition dynamics $P$, we know that $\overline{Q}^\mu(s,a) \in [0, 1]$ for all $s \in \SS$ and $a \in \AA$.
            Furthermore, when $s \in \safeset$, we have $c(s,a)$ and therefore $\overline{Q}^\mu(s,a) = \gamma \E_{s' | s, a}[\overline{Q}^\mu (s', \mu)] \in [0, \gamma]$.

        \item \textbf{Composite intervention}:
            For any $k \in \{1, \dots, K\}$, the following bound holds:
            \begin{align*}
            (\overline\BB^{\mu_{\min}} \overline{Q}_{\min})(s,a)
            &= c(s,a) + \gamma \E_{s' | s, a}[\overline{Q}_{\min}(s', \mu_{\min})] \\
            &\leq c(s,a) + \gamma \E_{s' | s, a} [\overline{Q}_{\min}(s', {\color{blue} \mu_k})] \\
            &\leq c(s,a) + \gamma \E_{s' | s, a} [{\color{blue} \overline{Q}_k}(s', \mu_k)] \\
            &\leq \overline{Q}_k(s,a) + \sigma_k \\
            &\leq \overline{Q}_k(s,a) + \sigma_{\max},
            \end{align*}
            where the first inequality comes from $\mu_{\min}$ being a minimizer of $\overline{Q}_{\min}$, and the second inequality from $\overline{Q}_{\min}$ being a pointwise minimum of $\{\overline{Q}_k\}_{k=1}^K$.
            Since this holds for every $k$, we conclude:
            \begin{align*}
                (\overline\BB^{\mu_{\min}} \overline{Q}_{\min})(s,a)
                &\leq \min_k \left[ \overline{Q}_k (s,a) + \sigma_{\max} \right] \\
                &= \min_k \overline{Q}_k(s,a) + \sigma_{\max} \\
                &= \overline{Q}_{\min}(s,a) + \sigma_{\max},
            \end{align*}
            which establishes the Bellman bound holds.
            Finally, since each $\overline{Q}_k$ satisfies $\overline{Q}_k \in [0, \gamma]$ on $\safeset$, we conclude that $\overline{Q}_{\min}$ has the same range.
            Therefore, $\GG$ is $\sigma_{\max}$-admissible.

        \item \textbf{Value iteration}:
            Define shortcuts $\overline{Q}_k \coloneqq \overline\TT^k \overline{Q}$, where $\overline{Q}_0 = \overline{Q}$.

            We first show that, by policy improvement, we have $\overline{Q}_k (s,a) \leq \overline{Q}_{k-1} (s,a) + \gamma^{k-1} \sigma$ on $\safeset \times \AA$.
            We do this by induction.
            First, we see that:
            \begin{align*}
                \overline{Q}_1 (s,a)
                &= \overline{\TT}\, \overline{Q}_0(s,a) \\
                &= c(s,a) + \gamma \E_{s' | s, a} \left[ \min_{a'} \overline{Q}_0(s', a') \right] \\
                &= c(s,a) + \gamma \E_{s' | s, a} \left[ \min_{a'} {\color{blue} \overline{Q}}(s', a') \right] \\
                &\leq c(s,a) + \gamma \E_{s' | s, a} \left[ \overline{Q}(s', {\color{blue} \mu}) \right] \\
                &\leq \overline{Q}(s,a) + \sigma \\
                &= \overline{Q}_0(s,a) + \sigma .
            \end{align*}
            Now suppose $\overline{Q}_\kappa (s,a) \leq \overline{Q}_{\kappa-1} (s,a) + \gamma^{\kappa - 1} \sigma$ holds on $\safeset \times \AA$ for some $\kappa$.
            Therefore,
            \begin{align*}
                \overline{Q}_{\kappa+1}(s,a)
                &= \overline{\TT}\, \overline{Q}_\kappa (s,a) \\
                &= c(s,a) + \gamma \E_{s' | s, a} \left[ \min_{a'} \overline{Q}_\kappa (s', a') \right] \\
                &\leq c(s,a) + \gamma \E_{s' | s, a} \left[ \min_{a'} \overline{Q}_{\kappa-1} (s', a')  \right] + \gamma^\kappa \sigma \\
                &= \overline{\TT}\, \overline{Q}_{\kappa-1}(s,a) + \gamma^\kappa \sigma \\
                &= \overline{Q}_\kappa (s,a) + \gamma^\kappa \sigma .
            \end{align*}

        Using this inequality, we now show that $\GG^k = (\overline{Q}_k, \mu^k, \eta)$ is indeed $\gamma^k \sigma$-admissible:
        \begin{align*}
            \overline{Q}_k(s,a)
            &= \overline{\TT}\, \overline{Q}_{k-1} (s,a) \\
            &= c(s,a) + \gamma \E_{s' | s, a} \left[ \min_{a'} \overline{Q}_{k-1} (s', a') \right] \\
            &\geq c(s,a) + \gamma \E_{s' | s, a} \left[ \min_{a'} \overline{Q}_k (s', a') \right] - \gamma^k \sigma \\
            &= \overline{\TT}\, \overline{Q}_k (s,a) - \gamma^k \sigma \\
            &= \overline{\BB}^{\mu^k} \overline{Q}_k (s,a) - \gamma^k \sigma,
        \end{align*}
        where the inequality was used in the third line.
        This establishes the Bellman bound holds.

        We prove that $\overline{Q}_k \in [0, \gamma]$ on $\safeset$ by induction.
        Clearly, $\overline{Q}_0 = \overline{Q} \in [0, \gamma]$ on $\safeset$ since $\GG$ is $\sigma$-admissible.
        Now suppose $\overline{Q}_\kappa \in [0, \gamma]$ on $\safeset$ for some $\kappa$.
        Then, for any $s \in \safeset$ and $a \in \AA$, we have $\overline{Q}_{\kappa+1}(s,a) = \gamma \E_{s' | s, a} [\min_{a'} \overline{Q}_\kappa (s,a)] \in [0, \gamma]$.
        Therefore, $\GG^k$ is $\gamma^k \sigma$-admissible.

    \item \textbf{Optimal intervention}:
        This is a special case of case 1.

    \item \textbf{Approximation}:
        The following holds on $\safeset \times \AA$:
        \begin{align*}
            \hat{Q}(s,a)
            &= \hat{Q}(s,a) - \overline{Q}(s,a) + \overline{Q}(s,a) \\
            &\geq -\delta + (\overline\BB^\mu \overline{Q})(s,a) - \sigma \\
            &= -\delta - \sigma + c(s,a) + \gamma \E_{s' | s, a} [\overline{Q}(s', \mu)] \\
            &\geq -\delta - \sigma + c(s,a) + \gamma \E_{s' | s, a} [\hat{Q}(s', \mu) - \delta] \\
            &= -\delta - \sigma - \gamma\delta + \overline\BB^\mu \hat{Q}(s,a) .
        \end{align*}
        That is, $\overline\BB^\mu \hat{Q}(s,a) \leq \hat{Q}(s,a) + \sigma + (1+\gamma)\delta$.
        Therefore, $\hat\GG = (\hat{Q}, \mu, \eta)$ is $(\sigma + (1+\gamma)\delta)$-admissible.
    \end{enumerate}
\end{proof}

\subsubsection{Safety Guarantee of Shielded Policy}
Before proving \cref{th:shielded policy is safe}, we prove two lemmas,
one showing that the average advantage of a shielded policy satisfies the intervention threshold (\cref{lm:shielded policy advantage inequality})
and the other stating that the cost-value function is equal to the expected occupancy of the unsafe set (\cref{lm:value function and expected occupancy of unsafe set}).

\begin{lemma} \label{lm:shielded policy advantage inequality}
    For some policy $\pi$ and intervention rule $\GG = (\overline{Q}, \mu, \eta)$, let $\pi' \coloneqq \GG(\pi)$ and $\overline{A}(s, a) \coloneqq \overline{Q}(s,a) - \overline{Q}(s, \mu)$.
    Then, $\overline{A}(s, \pi') \leq \eta$ for any $s \in \safeset$.
\end{lemma}
\begin{proof}
    We use the definition of $\pi'$ (in \eqref{eq:shielded policy}), the facts that $\overline{A}(s, \mu) = 0$, and that $(s,a)\notin \II$ if and only if $\overline{A}(s,a) \leq \eta$.
    The following then holds:
    \begin{align*}
        \overline{A}(s, \pi')
        &= \sum_{a \in \AA} \pi'(a|s) \overline{A}(s,a) \\
        &= \smashoperator{\sum_{a : (s,a) \notin \II}} \pi(a|s) \overline{A}(s,a) + w(s) \sum_{a \in \AA} \mu(a|s) \overline{A}(s,a) \\
        &\leq \eta \smashoperator{\sum_{a : (s,a) \notin \II}} \pi(a|s) + w(s) \overline{A}(s, \mu) \\
        &\leq \eta \cdot 1 + w(s) \cdot 0 \\
        &= \eta .
    \end{align*}
\end{proof}

\begin{lemma} \label{lm:value function and expected occupancy of unsafe set}
    For any policy $\pi$,
    \[
        \E_{s \sim d^\pi} [\one\{s \in \{\violation, \absorbing\}\}] = \overline{V}^\pi(d_0) .
    \]
\end{lemma}
\begin{proof}
    We know from the definition of the cost function that $\overline{V}^\pi(d_0) = \frac{1}{1-\gamma} \E_{s \sim d^\pi}[\one\{s = \violation\}]$.
    From the absorbing property of $\unsafeset$, we have $\E_{s \sim d^\pi}[\one\{s = \absorbing\}] = \frac{\gamma}{1-\gamma} \E_{s \sim d^\pi}[\one\{s = \violation\}]$.
    We can then derive
    \begin{align*}
        \E_{s \sim d^\pi} [\one\{s \in \{\violation, \absorbing\}\}]
        &= \E_{s \sim d^\pi} [\one\{s = \violation\}] + \E_{s \sim d^\pi} [\one\{s = \absorbing\}] \\
        &= \frac{1}{1-\gamma} \E_{s \sim d^\pi} [\one\{s = \violation\}] \\
        &= \overline{V}^\pi(d_0) .
    \end{align*}
\end{proof}

We now prove the safety guarantee of the shieled policy $\pi'$.

\IntervenedPolicyIsSafe*
\begin{proof}

    \begin{align*}
        \overline{Q}(\violation,a)=1 \qquad \text{and} \qquad \overline{Q}(\absorbing,a)=0 \qquad \text{for all}~ a \in \AA .
    \end{align*}

    Define $\overline{V}(s) \coloneqq \overline{Q}(s,\mu)$.
    Since
    $ c(s,a) + \gamma \E_{s'|s,a}[ \overline{V}(s')] = \overline{V}(s)$ when $s \in \{\violation, \absorbing\}$, we can use the performance difference lemma (\cref{lm:pdl}) to derive
    \begin{align*}
        \overline{V}^{\pi'}(d_0) - \overline{Q}(d_0,\mu)
        &= \frac{1}{1-\gamma} \E_{(s, a) \sim d^{\pi'}} [ c(s,a) + \gamma \E_{s'|s,a}[ \overline{V}(s')] - \overline{V}(s) ] \\
        &= \frac{1}{1-\gamma} \E_{(s, a) \sim d^{\pi'}} [ \left(  c(s,a) + \gamma \E_{s'|s,a}[ \overline{V}(s')] - \overline{V}(s) \right)  \one\{s  \not\in \{\violation, \absorbing\} \} ] \\
        &\leq \frac{1}{1-\gamma} \E_{(s, a) \sim d^{\pi'}} [ \left(  \min\{\sigma,\gamma\} + \overline{Q}(s,a) - \overline{V}(s) \right) \one\{s  \not\in \{\violation, \absorbing\} \} ] \\
        &\leq  \frac{ \min\{\sigma,\gamma\}+\min\{\eta,\gamma\}  }{1-\gamma}  \E_{s \sim d^{\pi'}} [ \one\{s  \not\in \{\violation, \absorbing\} \} ] \\
        &= \frac{ \min\{\sigma,\gamma\} + \min\{\eta,\gamma\} }{1-\gamma} \overline{V}^{\pi'}(d_0),
    \end{align*}
    where the first inequality comes from $\overline{Q}$ being $\sigma$-admissible and $\gamma$-admissible,
    the second inequality from $\overline{A}(s, \pi') \leq \eta$~(\cref{lm:shielded policy advantage inequality}) and $\overline{A}(s, \pi') \leq \gamma$~(\cref{def:admissible intervention}) for $s \notin \{\violation, \absorbing\}$,
    and the last equality from~\cref{lm:value function and expected occupancy of unsafe set}.

    Therefore, after some algebraic rearrangement,
    \begin{align*}
        \overline{V}^{\pi'}(d_0)
        &\leq \frac{ (1-\gamma)\overline{Q}(d_0,\mu) + \min\{\sigma,\gamma\} +\min\{\eta,\gamma\}  }{1-\gamma + \min\{\sigma,\gamma\} + \min\{\eta,\gamma\}}\\
        &\leq \overline{Q}(d_0,\mu) + \frac{\min\{\sigma,\gamma\}+\min\{\eta,\gamma\} }{1-\gamma}\\
        &\leq \overline{Q}(d_0,\mu) + \frac{\min\{\sigma+\eta,2\gamma\} }{1-\gamma} .
    \end{align*}

\end{proof}

\subsubsection{An Optimal Intervention Rule}
First, we show that every state-action pair visited by $\pi'$ will not have an advantage function lower than that of the optimal policy for $\overline{\MM}$.

\begin{restatable}{lemma}{OptimalInterventionLemma} \label{lm:optimal intervention lemma}
    Let $\overline\pi^*$ be an optimal policy for $\overline{\MM}$, $\overline{Q}^*$ be its state-action value function, and $\overline{V}^*$ be its state value function.
    Let $G_0 = \{(\overline{Q}, \mu, 0) : (\overline{Q}, \mu, 0)~\text{is admissible}, \; \overline{Q}(d_0, \mu) = \overline{V}^*(d_0)\}$ be a subset of admissible intervention rules with a threshold of zero and average $\overline{Q}$ that matches $\overline{V}^*$.
    Define $\overline{A}^*(s,a) = \overline{Q}^*(s,a) - \overline{Q}^*(s, \overline{\pi}^*)$ as the advantage function of the optimal policy.
    For some intervention rule $\GG \in G_0$ and policy $\pi$, let $\pi' = \GG(\pi)$.

    Then, the inequality $\overline{A}(s,a) \geq \overline{A}^*(s,a)$ holds for all $a \in \AA$ almost surely over the distribution $d^{\pi'}(s)$.
\end{restatable}
\begin{proof}
    First, we show by induction that running $\pi'$ starting from $d_0$ results in the agent staying in the subset $\SS_\GG = \{s \in \SS : \overline{Q}(s, \mu) = \overline{V}^*(s)\}$.

    For $t = 0$, consider some $s_0 \sim d_0$.
    We observe from admissibility of $\GG$ and~\cref{th:pessimistic estimate} that $\overline{Q}(s,a) \geq \overline{Q}^\mu(s,a) \geq \overline{V}^*(s)$ on $\SS \times \AA$.
    Since $\overline{Q}(d_0, \mu) = \overline{V}^*(d_0)$, we conclude that $\overline{Q}(s_0, \mu) = \overline{V}^*(s_0)$.
    Therefore, $s_0 \in \SS_\GG$ almost surely over $d_0$.

    Now suppose the agent is in $\SS_\GG$ with probability one at some time step $t$.
    Consider some $s_t \sim d_t^{\pi'}$ (observing that $s_t \in \SS_\GG$).
    We assume that $s_t \in \safeset$ (otherwise, the below is trivially true as there is no intervention outside $\safeset$).
    By~\cref{lm:shielded policy advantage inequality} and admissibility, we can derive:
    \begin{align*}
        0 = \eta
        &\geq \overline{A}(s_t, \pi') \\
        &= \overline{Q}(s_t, \pi') - \overline{Q}(s_t, \mu) \\
        &\geq c(s_t, \pi') + \gamma \E_{s_{t+1} \sim \PP | s_t, \pi'} [\overline{Q}(s_{t+1}, \mu)] - \overline{Q}(s_t, \mu) \\
        &= \gamma \E_{s_{t+1} | s_t, \pi'} [\overline{Q}(s_{t+1}, \mu)] - \overline{Q}(s_t, \mu) \\
        &= \gamma \E_{s_{t+1} | s_t, \pi'} [\overline{Q}(s_{t+1}, \mu)] - \overline{V}^*(s_t) \\
        &= \gamma \E_{s_{t+1} | s_t, \pi'} [\overline{Q}(s_{t+1}, \mu)] - \gamma \E_{s_{t+1} | s_t, \overline{\pi}^*} [\overline{V}^*(s_{t+1})],
    \end{align*}
    where the second and fourth equalities are due to $s_t \in \safeset$, and the third equality is due to $s_t \in \SS_\GG$.
    Notice also, since $s_t \in \safeset$, we have
    \begin{align*}
    \gamma \E_{s_{t+1}|s_t,\pi'}[\overline{V}^*(s_{t+1})]  = \overline{Q}^*(s_t, \pi') \geq \overline{Q}^*(s_t, \overline{\pi}^*) = \gamma \E_{s_{t+1}|s_t,\overline{\pi}^*}[\overline{V}^*(s_{t+1})] .
    \end{align*}
    Therefore, combining the two inequalities above, we have
    \begin{align*}
         \E_{s_{t+1}|s_t, \pi'}[\overline{V}^*(s_{t+1})]
         \geq \E_{s_{t+1}|s_t,\pi'}[\overline{Q}(s_{t+1},\mu)]  .
    \end{align*}
    Since $\overline{Q}(s,a) \geq \overline{V}^*(s)$ on $\SS \times \AA$, by the same argument we made for $s_0$, we conclude $\overline{Q}(s_{t+1},\mu)=\overline{V}^*(s_{t+1})$ with probability one. 
    Therefore, the agent stays in the subset $\SS_\GG$.
    
    With this property in mind, let $s \sim d^{\pi'}$.
    Then the following holds for all $a \in \AA$:
    \begin{align*}
        \overline{A}(s,a) 
        &= \overline{Q}(s,a) - \overline{Q}(s,\mu)\\
        &= \overline{Q}(s,a) -  \overline{Q}^*(s,\overline{\pi}^*) \\
        &\geq \overline{Q}^*(s,a) - \overline{Q}^*(s,\overline{\pi}^*) = \overline{A}^*(s,a),
    \end{align*}
    where the second equality is due to $\overline{Q}(s,\mu) = \overline{V}^*(s) = \overline{Q}^*(s,\overline{\pi}^*)$ on $\SS_\GG$.
\end{proof}

\OptimalIntervention*
\begin{proof}
    Let $\xi = (s_0, a_0, s_1, a_1, \dots)$ be any trajectory that has non-zero probabilty in the trajectory distribution $\widetilde{\rho}^\pi$ of $\pi$ on $\widetilde{\MM}$. 
    Let $\II$ and $\II^*$ be the intervention sets of $\GG$ and $\GG^*$, respectively.
    Suppose for some $t$ that $(s_t, a_t) \in \II$.
    We know for $\tau\geq t+1$ that $s_{\tau} = \intervened$.
    In addition, by \cref{lm:optimal intervention lemma}, we have $\overline{A}^*(s_\tau, a_\tau) \leq \overline{A}(s_\tau, a_\tau) \le 0$ for any $\tau \in [0, t-1]$, so $(s_\tau, a_\tau) \notin \II^*$.
    Therefore, the sub-trajectory $(s_\tau, a_\tau)$ with $\tau \in \{0, 1, \dots, t\}$ also has a non-zero probability in $\widetilde{\MM}^*$.
    By this argument, every sub-trajectory in $\SS\times\AA$ with non-zero probability in $\widetilde{\MM}$ also has non-zero probability in $\widetilde{\MM}^*$.
    The final thesis follows from defining the state-action distributions through averaging the trajectory distributions.
\end{proof}

\subsection{Proof for Absorbing MDP in \cref{sec:analysis of absorbing MDPs}}

We derive some properties of the Bellman operator of the absorbing MDP.
\begin{lemma} \label{lm:Bellman operators}
    For a policy $\pi$, let $(\BB^\pi Q)(s,a) \coloneqq r(s,a)+ \gamma \E_{s' | s, a}[Q(s',\pi)]$ denote the Bellman operator of $\pi$ in $\MM$;
    similarly define $\widetilde{\BB}^\pi$ for $\widetilde{\MM}$.
    Let $Q:\widetilde{\SS}\times\AA\to\R$ be some function satisfying $Q(\intervened, a) = 0$ for all $a \in \AA$.

    \begin{enumerate}
        \item The Bellman operator in $\widetilde{\MM}$ can be written as
            \begin{equation} \label{eq:bellman operator in surrogate MDP}
                (\widetilde{\BB}^\pi Q) (s,a) =
                \begin{cases}
                    (\BB^\pi Q)(s,a) \cdot \one\{(s,a) \notin \II\} + \widetilde{R} \cdot \one\{(s,a) \in \II\}, & (s, a) \in \SS \times \AA \\
                    0, & s = \intervened .
                \end{cases}
            \end{equation}
        \item The following holds when the temporal-difference operator $\widetilde{\DD}^\pi$ for $\widetilde\MM$ is applied to the policy's state-action value function $Q^\pi$ for $\MM$:
            \begin{align}
                \left(\widetilde{R} - \frac{1}{1-\gamma}\right) \one\{(s,a) \in \II\}
                \leq (\widetilde\DD^\pi Q^\pi)(s,a)
                &\leq \widetilde{R} \, \one\{(s,a) \in \II\}
                \quad \text{for all}~ (s,a) \in \SS \times \AA \label{eq:range of TD operator} \\
                (\widetilde{\DD}^\pi Q^\pi) (\intervened, a) &= 0, \label{eq:equality at absorbing state}
            \end{align}
            where the definition of $Q^\pi$ is extended to $\intervened$ as $Q^\pi(\intervened, a) = 0$.
    \end{enumerate}
\end{lemma}

\begin{proof}
    For brevity, let $\Omega(s,a) = \one\{(s,a) \in \II\}$.
    \begin{enumerate}
        \item Since $Q(\intervened, \pi) = 0$, the following holds for any $(s, a) \in \SS \times \AA$:
            \begin{align*}
                (\widetilde{\BB}^\pi Q) (s,a)
                &= \widetilde{r}(s,a) + \gamma \E_{s'\sim\widetilde{P}|s,a}[ Q(s',\pi)] \\
                &= (1-\Omega(s,a)) \left(r(s,a) + \gamma \E_{s' \sim P|s,a}[Q(s', \pi)]\right) + \Omega(s,a) \cdot \widetilde{R} \\
                &= (1 - \Omega(s,a)) \cdot (\BB^\pi Q)(s,a) + \Omega(s,a) \cdot \widetilde{R}
            \end{align*}
            and
            \begin{align*}
                (\widetilde{\BB}^\pi Q) (\intervened,a) = 0 + \gamma Q(\intervened, \pi) = 0 .
            \end{align*}

        \item For \eqref{eq:range of TD operator}, using the fact that $(\BB^\pi Q)^\pi = Q^\pi$, the following applies on $\SS \times \AA$:
            \[
                (\widetilde{\DD}^\pi Q^\pi) (s,a) = (\widetilde{\BB}^\pi Q^\pi) (s,a) - Q^\pi(s,a) = \Omega(s,a) \cdot \left( \widetilde{R} - Q^\pi (s,a) \right) .
            \]
            Since $0 \leq Q^\pi(s,a) \leq \frac{1}{1-\gamma}$, we have
            \[
                \left(\widetilde{R} - \frac{1}{1-\gamma}\right) \Omega(s,a) \leq (\widetilde\DD^\pi Q^\pi)(s,a) \leq \widetilde{R} \, \Omega(s,a) .
            \]

            For the absorbing state in \eqref{eq:equality at absorbing state}, by the extended definition and the equality $(\widetilde{\BB}^\pi Q) (\intervened,a) = 0$, we have
            \[
                (\widetilde{\DD}^\pi Q^\pi) (\intervened, a) = (\widetilde{\BB}^\pi Q^\pi) (\intervened, a) - Q^\pi(\intervened, a) = 0 .
            \]
    \end{enumerate}
\end{proof}

\begin{restatable}{lemma}{IntervenedProbability}\label{lm:intervened probability}
For any policy $\pi$, $ P_\GG(\pi) =  \frac{1}{1-\gamma}\E_{(s,a) \sim \widetilde{d}^ \pi}[ \one\{ (s,a)\in\II \} ]$ .
\end{restatable}
\begin{proof}
    Notice that for any $h$,
    \begin{align*}
     \prob(\xi^h \cap \II \neq \varnothing \mid \pi, \MM )
     &= \prob(\xi^h \cap \II \neq \varnothing \mid \pi, {\color{blue} \widetilde{\MM}} )= \E_{\widetilde{\rho}^\pi} \left[ \sum_{t=0}^{h-1}   \one\{ (s_t,a_t)\in\II \}  \right] .
    \end{align*}
    By \cref{lm:discounted value as average value},
    \begin{align*}
         \frac{1}{1-\gamma}\E_{(s,a) \sim \widetilde{d}^\pi} [ \one\{ (s,a)\in\II \} ]
         &= \E_{\widetilde{\rho}^\pi} \left[ \sum_{t=0}^{\infty} \gamma^t  \one\{ (s_t,a_t)\in\II \} \right] \\
         &= (1-\gamma) \sum_{h=0}^\infty \gamma^h  \E_{\widetilde{\rho}^\pi} \left[ \sum_{t=0}^{h-1}   \one\{ (s_t,a_t)\in\II \}  \right] \\
         &= (1-\gamma) \sum_{h=0}^\infty \gamma^h  \prob(\xi^h \cap \II \neq \varnothing \mid \pi, \MM ) \\
         &= P_\GG(\pi) .
    \end{align*}
\end{proof}

Using the above results, we can bound the difference between the values of  the original and the absorbing MDPs.

\ValueOffset*
\begin{proof}[Proof of \cref{th:value offset}]
    First, extend the definition of $Q^\pi$ to $\intervened$ as $Q^\pi(\intervened, a) = 0$ for any $a \in \AA$.
    By \cref{lm:simulation lemma}, we have
    \[
        \widetilde{V}^\pi(d_0) - V^\pi(d_0) = \frac{1}{1-\gamma}\E_{(s,a) \sim \widetilde{d}^\pi}[(\widetilde{\DD}^\pi Q^\pi)(s, a)]
    \]
   By \cref{lm:Bellman operators}, we can derive
   \[
       \left(\widetilde{R} - \frac{1}{1-\gamma}\right) \frac{\E_{(s,a) \sim \widetilde{d}^\pi} [\one\{(s,a) \in \II\}]}{1-\gamma}
       \leq \widetilde{V}^\pi (d_0) - V^\pi (d_0)
       \leq \widetilde{R} \, \frac{\E_{(s,a) \sim \widetilde{d}^\pi} [\one\{(s,a) \in \II\}]}{1-\gamma} .
   \]
   Finally, substituting the equality from \cref{lm:intervened probability} and negating the inequality concludes the proof.
\end{proof}

Next we derive some lemmas, which will be later to used to show that when the intervention set is partial, the unconstrained reduction is effective.

\begin{lemma} \label{lm:safer policy}
    Let $\II \subset \safeset \times\AA$ be partial, and let $\FF = (\safeset \times \AA) \setminus \II$ be the state-action pairs that are not intervened.
    For an arbitrary policy $\pi$, define
    \begin{equation} \label{eq:non-intervened policy}
        \pi_f(a|s) \coloneqq \pi(a|s) \one\{(s,a) \in \FF \} + f(s,a),
    \end{equation}
    where $f(s,a)$ is some arbitrary non-negative function which is zero on $\II$ and that ensures $\sum_{a\in\AA}\pi_f(a|s)=1$ for all $s \in \SS$.
    Define
    \begin{align*}
        \widetilde{J}^{\pi}_+ &\coloneqq \frac{1}{1-\gamma} \E_{(s,a) \sim \widetilde{d}^\pi}[ r(s,a) \cdot \one\{ (s,a) \in \FF \} ]
        \\
        \widetilde{J}^{\pi}_- &\coloneqq \frac{1}{1-\gamma} \E_{(s, a) \sim \widetilde{d}^\pi}[ \widetilde{R} \cdot \one\{ (s,a) \in \II \} ]
    \end{align*}
    as the expected returns in $\FF$ and $\II$, respectively.

    The following are true:
    \begin{enumerate}
        \item $\widetilde{V}^\pi(d_0) = \widetilde{J}^\pi_+ + \widetilde{J}^\pi_-$.
        \item $\widetilde{d}^{\pi_f}(s,a) \geq \widetilde{d}^\pi(s,a)$ for all $(s,a) \in \FF$.
        \item $\widetilde{J}^{\pi_f}_+ \geq \widetilde{J}^{\pi}_+$.
        \item $\E_{(s,a)\sim\widetilde{d}^{\pi_f}}[\one\{ (s,a)\in \II \}] = 0$, implying  $\widetilde{J}^{\pi_f}_- = 0$.
        \item $\widetilde{V}^{\pi_f}(d_0) \geq \widetilde{V}^\pi(d_0)$ whenever $\widetilde{R} \leq 0$.
            Furthermore, if $\widetilde{R} < 0$ and $\pi(a|s)>0$ for some $(s,a)\in\II$, then ${\widetilde{V}^{\pi_f}(d_0) > \widetilde{V}^\pi(d_0)}$.
    \end{enumerate}
\end{lemma}

\begin{proof}
    \begin{enumerate}
        \item This follows from the definition of $\widetilde r$ in \eqref{eq:modified reward}.
            
        \item Recall $\widetilde{d}^\pi(s,a) = (1-\gamma) \sum_{t=0}^\infty \gamma^t \widetilde{d}^\pi_t(s, a)$.
            To show the desired result, we show by induction that $\widetilde{d}^{\pi_f}_t(s,a) \ge \widetilde{d}^\pi_t(s, a)$ for all $(s,a) \in \FF$ and $t \geq 0$.
            For $t = 0$, by construction of $\pi_f$, we have $\pi_f(a|s) \geq \pi(a|s)$ for all $(s,a) \in \FF$ and therefore $\widetilde{d}^{\pi_f}_0 (s,a) \geq \widetilde{d}^\pi_0(s,a)$ for all $(s,a) \in \FF$.

            Now suppose that for some $t \geq 0$ the inequality $\widetilde{d}^{\pi_f}_t(s,a) \geq \widetilde{d}^\pi_t(s,a)$ holds for all $(s,a) \in \FF$.
            Then, for some $(s,a) \in \FF$, we can derive
            \begin{align*}
                \widetilde{d}^{\pi_f}_{t+1}(s,a)
                &= \pi_f(a|s) \sum_{(s_t, a_t) \in \SS \times \AA} \widetilde{P}(s|s_t, a_t) \widetilde{d}^{\pi_f}_t(s_t, a_t) \\
                &= \pi_f(a|s) \sum_{(s_t, a_t) \in \FF} P(s|s_t, a_t) \widetilde{d}^{\pi_f}_t(s_t, a_t) \\
                &\geq \pi(a|s) \sum_{(s_t, a_t) \in \FF} P(s|s_t,a_t) \widetilde{d}^\pi_t(s_t, a_t) \\
                &= \widetilde{d}^\pi_{t+1}(s,a),
            \end{align*}
            where we use the inductive hypothesis in the inequality.
            Thus, we have $\widetilde{d}^{\pi_f}(s,a) \geq \widetilde{d}^\pi(s,a)$ by summing over each time step.
            
        \item By statement 2, definition of $\widetilde{J}^\pi_+$, and non-negativity of the reward $r$, it follows that $\widetilde{J}^{\pi_f}_+ \geq \widetilde{J}^\pi_+$.

        \item This statement from the construction of $\pi_f$ and induction.
            First, we have $\widetilde{d}^{\pi_f}_0 (s,a) = 0$ for all $(s, a) \in \II$.
            Now suppose for some $t \geq 0$ that $\widetilde{d}^{\pi_f}_t(s,a) = 0$ for all $(s,a) \in \II$.
            We can see that $\widetilde{d}^{\pi_f}_{t+1}(s, a) = 0$ for all $(s,a) \in \II$ since $\pi_f$ never chooses actions such that $(s,a) \in \II$.
            
            Therefore, $\widetilde{d}^{\pi_f}(s,a) = 0$ for all $(s,a) \in \II$.
            By definition of $\widetilde{J}^\pi_-$, this allows us to conclude that $\widetilde{J}^\pi_- = 0$.

        \item Using statements 3 and 4, we conclude that
            \[
                \widetilde{V}^{\pi_f}(d_0)
                = \widetilde{J}^{\pi_f}_+ + \widetilde{J}^{\pi_f}_-
                \geq \widetilde{J}^\pi_+ + \widetilde{J}^\pi_-
                = \widetilde{V}^\pi(d_0) .
            \]
            The special case follows from observing that $\widetilde{J}^\pi_- < 0$ whenever $\pi(a|s) > 0$ for some $(s,a) \in \II$.
    \end{enumerate}
\end{proof}

\begin{restatable}{lemma}{OptimalPolicyNeverEntersI} \label{th:optimal policy never enters I}
    Let $\widetilde R$ be non-positive and $\widetilde{V}^*$ denote the optimal value function for $\widetilde{\MM}$.
    \begin{enumerate}
        \item For any policy $\pi$,
            \[
                \widetilde{V}^*(d_0) \geq \widetilde{J}^\pi_+ .
            \]
        \item There is an optimal policy $\widetilde\pi^*$ of $\widetilde\MM$ satisfying
            \begin{equation} \label{eq:optimal policy is not intervened}
                \E_{(s,a) \sim \widetilde{d}^{\widetilde\pi^*}}[\one\{(s,a) \in \II\}] = 0 .
            \end{equation}
        \item If $\widetilde R$ is negative, every optimal policy of $\widetilde\MM$ satisfies \eqref{eq:optimal policy is not intervened}.
    \end{enumerate}
\end{restatable}
\begin{proof}[Proof of \cref{th:optimal policy never enters I}]
    \begin{enumerate}
        \item Let the policy $\pi$ be arbitrary, and define $\pi_f$ using \eqref{eq:non-intervened policy}.
            The following then holds by \cref{lm:safer policy}:
            \[
                \widetilde{V}^*(d_0)
                \geq \widetilde{V}^{\pi_f}(d_0)
                = \widetilde{J}^{\pi_f}_+
                \geq \widetilde{J}^\pi_+ .
            \]

        \item Suppose that $\pi$ is an optimal policy of $\widetilde\MM$, and define $\pi_f$ using \eqref{eq:non-intervened policy}.
            Because $\widetilde{R}$ is non-positive, we know by \cref{lm:safer policy} and optimality of $\pi$ that $\widetilde{V}^{\pi_f}(d_0) = \widetilde{V}^\pi(d_0)$.
            Therefore, we can define an optimal policy as $\widetilde\pi^* = \pi_f$ and conclude by \cref{lm:safer policy} that $\E_{(s,a) \sim \widetilde{d}^{\widetilde\pi^*}}[\one\{(s,a) \in \II\}] = 0$.

        \item Suppose for the sake of contradiction there is an optimal policy $\widetilde\pi^*$ of $\widetilde\MM$ such that \eqref{eq:optimal policy is not intervened} does \emph{not} hold (i.e., it may take some $(s,a) \in \II$).
            By \cref{lm:safer policy}, we can construct some policy $\pi_f$ such that $\widetilde{V}^{\pi_f}(d_0) > \widetilde{V}^{\widetilde\pi^*}(d_0)$.
            This contradicts $\widetilde\pi^*$ being optimal, so every optimal policy of $\widetilde\MM$ must satisfy \eqref{eq:optimal policy is not intervened}.
    \end{enumerate}
\end{proof}

These results show that if the intervention set is partial and the penalty of being intervened is strict, then the optimal policy of the absorbing MDP would not be intervened.

\OptimalPolicyIsNotIntervened*
\begin{proof}[Proof of \cref{th:optimal policy is not intervened}]
    This directly follows from \cref{th:optimal policy never enters I,lm:intervened probability}.
\end{proof}

Below we derive some lemmas to show a near optimal policy of the absorbing MDP is safe. (We already proved above that the optimal policy of the absorbing MDP is safe).

\begin{lemma} \label{lm:policy safety bound}
    Let $\II \subset \SS\times\AA$ be partial (\cref{as:viable action}).
    Given some policy $\pi$, let $\pi'$ be the corresponding shielded policy defined in \eqref{eq:shielded policy}.
    Then, the following holds for any $h \ge 0$ in $\MM$:
    \begin{equation} \label{eq:policy safety bound}
        \prob(\violation \in \xi^h \mid \pi, \MM)
        \leq \prob(\violation \in \xi^h \mid \pi', \MM) + \prob(\xi^h \cap \II \ne \varnothing \mid \pi, \MM),
    \end{equation}
    where $\xi^h = (s_0, a_0, \dots, s_{h-1}, a_{h-1})$ is an $h$-step trajectory segment.
\end{lemma}
\begin{proof}
    First, we notice that $\pi'(a|s) \geq \pi(a|s)$ when $(s,a) \notin \II$, because $\pi'(a|s) = \pi(a|s) + w(s)\mu(a|s) \geq \pi(a|s)$.

    We bound the probability of $\pi$ violating a constraint in $\MM$ by introducing whether $\pi$ visits the intervention set:
    \begin{align*}
        \prob(\violation \in \xi^h \mid \pi, \MM)
        &= \prob(\violation \in \xi^h,\; \xi^h \cap \II = \varnothing \mid \pi, \MM) + \prob(\violation \in \xi^h, \; \xi^h \cap \II \neq \varnothing \mid \pi, \MM) \\
        &\leq \prob(\violation \in \xi^h, \; \xi^h \cap \II = \varnothing \mid \pi, \MM) + \prob(\xi^h \cap \II \neq \varnothing \mid \pi, \MM) .
    \end{align*}

    We now bound the first term.
    Let $\xi^h$ satisfy the event ``$\violation \in \xi^h, \; \xi^h \cap \II = \varnothing$'', and let $T$ be the time index such that $s_T = \violation$ in $\xi^h$.
    Then, the probability of this trajectory under $\pi$ and $\MM$ is
    \[
        d_0(s_0) \pi(a_0 | s_0) P(s_1 | s_0, a_0) \cdots \pi(a_{T-1} | s_{T-1}) P(s_T | s_{T-1}, a_{T-1}) .
    \]
    Since each $(s_t, a_t)$ is not in $\II$, we have $\pi(a_t|s_t) \leq \pi'(a_t|s_t)$ for each $(s_t, a_t)$ in $\xi^h$.
    Thus, the probability of this trajectory under $\pi$ and $\MM$ is upper bounded by its probability under ${\color{blue} \pi'}$ and $\MM$.
    Summing over each trajectory $\xi^h$ satisfying the event then yields:
    \[
        \prob(\violation \in \xi^h, \; \xi^h \cap \II = \varnothing \mid \pi, \MM)
        \leq \prob(\violation \in \xi^h, \; \xi^h \cap \II = \varnothing \mid {\color{blue} \pi'}, \MM) .
    \]

    We now complete the original bound:
    \begin{align*}
        \prob(\violation \in \xi^h \mid \pi, \MM)
        &\leq \prob(\violation \in \xi^h, \; \xi^h \cap \II = \varnothing \mid \pi, \MM) + \prob(\xi^h \cap \II \neq \varnothing \mid \pi, \MM) \\
        &\leq \prob(\violation \in \xi^h, \; \xi^h \cap \II = \varnothing \mid {\color{blue} \pi'}, \MM) + \prob(\xi^h \cap \II \neq \varnothing \mid \pi, \MM) \\
        &\leq \prob(\violation \in \xi^h \mid \pi', \MM) + \prob(\xi^h \cap \II \neq \varnothing \mid \pi, \MM) .
    \end{align*}
\end{proof}

\begin{restatable}{lemma}{DecompositionOfPolicySafety}\label{lm:decomposition of policy safety}
    For any policy $\pi$ and  $\II \subset \SS\times\AA$ that is partial, let $\pi'$ be the corresponding shielded policy.
    Then, the following safety bound holds:
    \[
        \overline{V}^\pi(d_0) \leq \overline{V}^{\pi'}(d_0) + \frac{1}{1-\gamma} \E_{(s,a) \sim \widetilde{d}^\pi} [\one\{(s,a) \in \II\}] .
    \]
\end{restatable}
\begin{proof}[Proof of \cref{lm:decomposition of policy safety}]
    Using \eqref{eq:policy safety bound} from \cref{lm:policy safety bound} and the fact that the probabilities can be expressed as expected sums of indicators:
    \begin{align*}
        \prob(\violation \in \xi^h \mid \pi, \MM) &= \E_{\rho^\pi} \left[ \sum_{t=0}^{h-1} \one\{s_t = \violation\} \right] \\
        \prob(\violation \in \xi^h \mid \pi', \MM) &= \E_{\rho^{\pi'}} \left[ \sum_{t=0}^{h-1} \one\{s_t = \violation\} \right] \\
        \prob(\xi^h \cap \II \neq \varnothing \mid \pi, \MM) &= \E_{\widetilde\rho^\pi} \left[ \sum_{t=0}^{h-1} \one\{(s_t, a_t) \in \II\} \right]
    \end{align*}

    Then, applying \cref{lm:discounted value as average value} results in the desired inequality.
\end{proof}

\PerformanceAndSafety*
\begin{proof}[Proof of \cref{th:performance and safety}]
    The performance bound follows from \cref{th:value offset}.
    \begin{align*}
        V^{\pi^*}(d_0) - V^\pi(d_0)
        &= V^{\pi^*}(d_0) - \widetilde{V}^{\pi^*}(d_0) + \widetilde{V}^{\pi^*}(d_0)- \widetilde{V}^{\pi}(d_0) + \widetilde{V}^\pi(d_0) - V^\pi(d_0) \\
        &\leq  \left(|\widetilde{R}| + \frac{1}{1-\gamma}\right)  P_\GG(\pi^*) + \widetilde{V}^{\pi^*}(d_0)- \widetilde{V}^{\pi}(d_0) - |\widetilde{R}| \, P_\GG(\pi) \\
        &\leq  \left(|\widetilde{R}| + \frac{1}{1-\gamma} \right)  P_\GG(\pi^*) + \widetilde{V}^{{\color{blue} *}}(d_0)- \widetilde{V}^{\pi}(d_0) \\
        &\leq  \left(|\widetilde{R}| + \frac{1}{1-\gamma} \right)  P_\GG(\pi^*) + \varepsilon .
    \end{align*}

    For the safety bound, we start with \cref{lm:decomposition of policy safety}:
    \[
        \overline{V}^\pi(d_0)
        \leq \overline{V}^{\pi'}(d_0) + \frac{1}{1-\gamma} \E_{(s,a) \sim \widetilde{d}^\pi} [\one\{(s,a) \in \II\}]
    \]
    We provide an upper bound on the second term on the right hand side above.
    Using the definition of $\widetilde{J}^\pi_-$ in \cref{lm:safer policy}, we derive that
    \begin{align*}
        \frac{\E_{\widetilde{d}^\pi} [\one\{(s,a) \in \II\}]}{1-\gamma}
        &= - \frac{\widetilde{J}^\pi_-}{|\widetilde R|} \\
        &= \frac{1}{|\widetilde R|} \left( -\widetilde{V}^\pi (d_0) + \widetilde{V}^* (d_0) + \widetilde{J}^\pi_+ - \widetilde{V}^* (d_0) \right) \\
        &\leq \frac{1}{|\widetilde R|} \left( \widetilde{V}^* (d_0) - \widetilde{V}^\pi (d_0) \right) = 
         \frac{\varepsilon}{|\widetilde R|},
    \end{align*}
    where the inequality is due to \cref{th:optimal policy never enters I}.

    Combine everything altogether:
    \begin{align*}
        \overline{V}^\pi(d_0)
        &\leq \overline{V}^{\pi'}(d_0) + \frac{\E_{\widetilde{d}^{\pi}}[\one\{ (s,a) \in \II \}] }{1-\gamma}\\
        &= \overline{V}^{\pi'}(d_0) + \frac{\varepsilon}{|\widetilde{R}|} .
    \end{align*}

\end{proof}

We now prove the main result of the paper.
\SuboptimalityBound*
\begin{proof}
    This is a direct result of \cref{th:performance and safety}.

    The performance suboptimality results from:
    \begin{align*}
        V^{\pi^*}(d_0)  -  V^{\hat\pi}(d_0)
        &\leq \left( |\widetilde{R}| + \frac{1}{1-\gamma} \right) P_\GG(\pi^*) + \varepsilon \\
        &\leq \left(  1 + \frac{1}{1-\gamma}  \right) P_\GG(\pi^*) + \varepsilon \\
        &= \frac{2-\gamma}{1-\gamma}  P_\GG(\pi^*) + \varepsilon \\
        &\leq \frac{2}{1-\gamma} P_\GG(\pi^*) + \varepsilon .
    \end{align*}

    For the safety bound,
    \begin{align*}
        \overline{V}^{\hat\pi}(d_0)
        &\leq \overline{V}^{\GG(\hat\pi)}(d_0) + \varepsilon \\
        &\leq \overline{Q}(d_0, \mu) + \frac{\min\{\sigma + \eta, 2\gamma\}}{1-\gamma} + \varepsilon,
    \end{align*}
    where the second inequality follows from \cref{th:shielded policy is safe} and $\varepsilon$-suboptimality of $\hat\pi$ in $\widetilde\MM$.
\end{proof}

\section{Additional Discussion of \alg}

\subsection{Necessity of the partial property} \label{sec:non-partial}

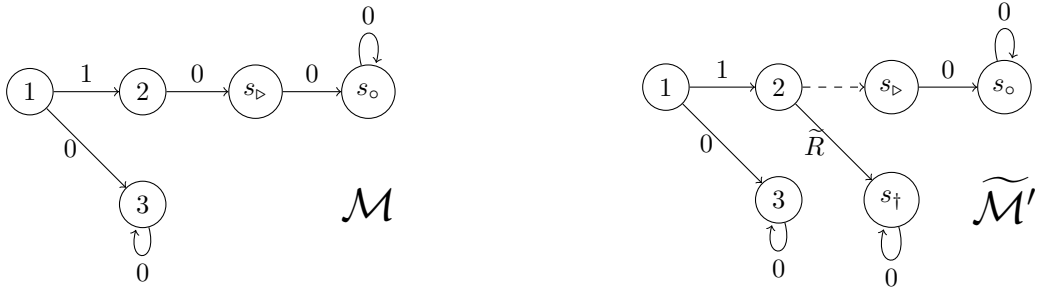
\begin{figure}[h!]
    \centering
    \begin{subfigure}{0.49\textwidth}
        \centering
        \begin{tikzpicture}[->, node distance=1.5cm]
  \node[draw, circle] (1) {$1$};
  \node[draw, circle, right of=1] (2) {$2$};
  \node[draw, circle, below of=2] (3) {$3$};
  \node[draw, circle, right of=2] (violation) {\violation};
  \node[draw, circle, right of=violation] (absorbing) {\absorbing};
  \node[below of=absorbing] (M) {{\LARGE $\MM$}};

  \path
  (1) edge node [above] {$1$} (2)
  (1) edge node [left] {$0$} (3)
  (2) edge node [above] {$0$} (violation)
  (3) edge [loop below] node {$0$} (3)
  (violation) edge node [above] {$0$} (absorbing)
  (absorbing) edge [loop above] node {$0$} (absorbing);
\end{tikzpicture}
    \end{subfigure}
    \begin{subfigure}{0.49\textwidth}
        \centering
        \begin{tikzpicture}[->, node distance=1.5cm]
  \node[draw, circle] (1) {$1$};
  \node[draw, circle, right of=1] (2) {$2$};
  \node[draw, circle, below of=2] (3) {$3$};
  \node[draw, circle, right of=2] (violation) {\violation};
  \node[draw, circle, right of=violation] (absorbing) {\absorbing};
  \node[draw, circle, right of=3] (intervened) {\intervened};
  \node[below of=absorbing] (M) {{\LARGE $\widetilde\MM'$}};

  \path[every path/.style={dashed}]
  (2) edge node [above] {} (violation);

  \path
  (1) edge node [above] {$1$} (2)
  (1) edge node [left] {$0$} (3)
  (3) edge [loop below] node {$0$} (3)
  (violation) edge node [above] {$0$} (absorbing)
  (absorbing) edge [loop above] node {$0$} (absorbing)
  (2) edge node [left] {$\widetilde R$} (intervened)
  (intervened) edge [loop below] node {$0$} (intervened);
  
\end{tikzpicture}
    \end{subfigure}
    \caption{A simple example illustrating a non-partial intervention. Edge weights correspond to rewards. If $\widetilde{R} > -1/\gamma$, the optimal policy in $\widetilde\MM'$ will always go into the intervention set.}
    \label{fig:non-partial intervention}
\end{figure}

We highlight that the subset $\II$ being \emph{partial} (\cref{def:admissible intervention}) is crucial for the unconstrained MDP reduction behind \alg.
If we were to construct an absorbing MDP $\widetilde{\MM}'$ described in \cref{sec:absorbing mdp} using an arbitrary non-partial subset $\II' \subseteq \SS\times\AA$, then the optimal policy of  $\widetilde{\MM}'$ can still enter $\II'$ when $\widetilde{R}>-\infty$, because the optimal policy of $\widetilde{\MM}'$ can use earlier rewards to make up for the penalty incurred in $\II'$.

To see this, consider the toy MDP $\MM$ shown in \cref{fig:non-partial intervention}.
Since there is no alternative action available at state $2$, the intervention illustrated in $\widetilde\MM'$ is \emph{not} partial.
Suppose $\widetilde{R} > -1/\gamma$.
Then, in $\widetilde\MM'$, a policy choosing to transition from $1$ to $3$ has a value of $0$, and a policy choosing to transition from $1$ to $2$ has a value of $1 + \gamma \widetilde{R} > 0$.
Therefore, the optimal policy will transition from $1$ to $2$ and go into the non-partial intervention set $\II'$.
Once applied to the original MDP $\MM$, this policy will always go into the unsafe set.

One might think generally it is possible to set $\widetilde{R}$ to be negative enough to ensure the optimal policy will never go into the intervention set, which is indeed true for the counterexample above.
But we remark that we need to set $\widetilde{R}$ to be arbitrarily large (in the negative direction) for general problems, which can cause high variance issues in return or gradient estimation~\citep{shalev2016safe}.
Because of the discount factor $\gamma<1$, the negative reward stemming from the absorbing state will be at most $\gamma^T \widetilde{R}$, where $T$ is the time step that the system enters $\II'$.
For a fixed and finite $\widetilde R$, we can then extend the above MDP construction to let the agent go through a long enough chain after transitioning from $1$ to $2$ so that the resultant value satisfies $1 + \gamma^T \widetilde{R} > 0$.
Like the example above, this path would be the only path with positive reward, despite intersecting the intervention set.
Therefore, the optimal policy of $\widetilde{\MM}'$ will enter $\II'$.

\subsection{Bias of \alg}

In \cref{th:main theorem}, we give a performance guarantee of \alg 
\[
    V^*(d_0) - V^{\hat\pi}(d_0) \leq \frac{2}{1-\gamma} P_\GG(\pi^*) + \varepsilon .
\]
It shows that \alg has a bias $P_{\GG}(\pi^*)\in[0,1]$, which is the probability that the optimal policy $\pi^*$ would be intervened by the advantage-based intervention rule.
Here we discuss special cases where this bias vanishes.

The first special case is when the original problem is unconstrained (i.e., \eqref{eq:CMDP formulation} has a trivial constraint with $\delta=1$). In this case, we can set the threshold $\eta \geq \gamma$ in \alg to turn off the intervention, and \alg returns the optimal policy of the MDP $\MM$ when the base RL algorithm can find one.

Another case is when $\pi^*$ is a perfect safe policy, i.e., $\overline{V}^{\pi^*}(d_0)=0$ and we run \alg with the intervention rule $\GG^* = (\overline{Q}^*, \overline{\pi}^* , 0)$ (\cref{th:optimal intervention}). Similar to the proof of \cref{lm:optimal intervention lemma}, one can show that running $\pi^*$ would not trigger the intervention rule $\GG^*$ and therefore the bias $P_{\GG^*}(\pi^*)$ is zero.

However, we note that generally the bias $P_{\GG}(\pi^*)$ can be non-zero.

\section{Experimental Details}
\subsection{Point Robot} \label{app:point robot}
\begin{figure}[h!]
    \centering
    \includegraphics[width=0.3\textwidth]{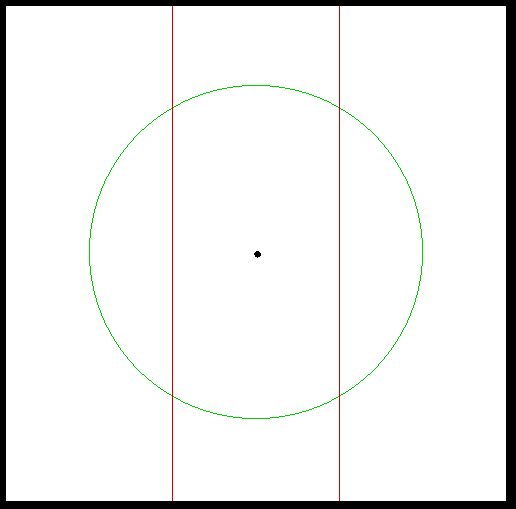}
    \caption{The point environment.
    The black dot corresponds to the agent, the green circle to the desired path, and the red lines to the constraints on the horizontal position.
    The vertical constraints are outside of the visualized environment.}
    \label{fig:point environment}
\end{figure}
This environment (\cref{fig:point environment}) is a simplification of the point environment from \citep{achiam2017constrained}.
The state is $s = (x, y, \dot x, \dot y)$, where $(x, y)$ is the x-y position and $(\dot x, \dot y)$ is the corresponding velocity.
The action $a = (a_x, a_y)$ is the force applied to the robot (each component has maximum magnitude $a_{\max}$).
The agent has some mass $m$ and can achieve maximum speed $v_\mathrm{max}$.
The dynamics update (with time increment $\Delta t$) is:
\begin{align*}
(x_{t+1}, y_{t+1}) &= (x_t, y_t) + (\dot{x}_t, \dot{y}_t) \Delta t + \frac{1}{2m} a_t \Delta t^2 \\
(\dot{x}_{t+1}, \dot{y}_{t+1}) &= \texttt{clip-norm}\left( ((\dot{x}_t, \dot{y}_t) + \frac{1}{m} a_t \Delta t,\, v_\mathrm{max} \right),
\end{align*}
where $\texttt{clip-norm}(u, c)$ scales $u$ so that its norm matches $c$ if $\norm{u} > c$.
The reward corresponds to following a circular path of radius $R^*$ at a high speed and the safe set to staying within desired positional bounds $x_\mathrm{max}$ and $y_\mathrm{max}$:
\begin{align*}
r(s, a) &= \frac{(\dot x, \dot y) \cdot (-y, x)}{1 + \abs{\norm{(x, y)} - R^*}} \\
\SS_\mathrm{safe} &= \{s \in \SS : \abs{x} \le x_\mathrm{max} ~\text{and}~ \abs{y} \le y_\mathrm{max} \}
\end{align*}

For our experiments, we set these parameters to $m = 1$, $v_\mathrm{max} = 2$, $a_{\max} = 1$, $\Delta t = 0.1$, $R^* = 5$, $x_\mathrm{max} = 2.5$, and $y_\mathrm{max} = 15$.

For this environment, we also consider a \emph{shaped} cost function $\hat{c}(s,a)$ which is a function of the \emph{distance} of the state $s$ to the boundary of the unsafe set, denoted by $\mathrm{dist}(s, \unsafeset)$.
Here, $\unsafeset$ denotes the 2D unsafe region in this environment (i.e., those outside the vertical lines in~\cref{fig:point environment}).
Note that in the theoretical analysis $\unsafeset$ is abstracted into $\{\violation, \intervened\}$.

For the point environment, the distance function is $\mathrm{dist}(s, S_\mathrm{unsafe}) = \max\{0, \min\{x_{\max}-x, x_{\max}+x, y_{\max}-y, y_{\max}+y\}\}$.
For some constant $\alpha \ge 0$, the cost function is defined as a hinge function of the distance:
\begin{equation}
\hat{c}(s, a) =
\begin{cases}
    \one\{\mathrm{dist}(s, \unsafeset) = 0\}, & \alpha = 0 \\
    \max \left\{ 0, 1 - \frac{1}{\alpha}\mathrm{dist}(s, \unsafeset) \right\}, & \text{otherwise.}
\end{cases}
\label{eq:hinge}
\end{equation}
We note that $\hat{c}$ is an upper bound for $c$ if $\alpha > 0$ and $\hat{c} = c$ if $\alpha = 0$.
We shape the cost here to make it continuous, so that the effects of approximation bias is smaller than that resulting from a discontinuous cost (i.e., the original indicator function).

\textbf{Intervention Rule:}
The backup policy $\mu$ applies a decelerating force (with component-wise magnitude up to $a_{\max}$) until the agent has zero velocity.
Our experiments consider the following approaches to construct $\overline{Q}$:
\begin{itemize}
    \item \textbf{Neural network approximation}:
        We construct a dataset of points mapping states and actions to state-action values $\overline{Q}^\mu$ by picking some state and action in the MDP, executing the action from that state, and then continuing the rollout with the backup policy $\mu$ to find the empirical state-action value with respect to the shaped cost function $\hat c$.
        Our dataset consists of $10^7$ points resulting from a uniform discretization of the state-action space.
        We apply a similar method to form a dataset for the state values $\overline{V}^\mu$.

        We then train four networks (two to independently approximate $\overline{Q}^\mu$, and two for $\overline{V}^\mu$), where each network has three hidden layers each with $256$ neurons and a ReLU activation.
        The predicted advantage is $\overline{A}(s, a) = \max\{\overline{Q}_1(s,a), \overline{Q}_2(s,a)\} - \min\{\overline{V}_1(s), \overline{V}_2(s)\}$, where we apply the pessimistic approach from~\citep{thananjeyan2020recovery} to prevent overestimation bias.

    \item \textbf{Model-based evaluation}:
        Here, we have access to a model of the robot where all parameters match the real environment except possibly the mass $\hat m$.
        We refer to the modeled transition dynamics as $\hat\PP$ and the resulting trajectory distribution under $\mu$ as $\hat\rho^\mu$.
        The function $\overline{Q}$ is then set to be the model-based estimate of $\overline{Q}^\mu$ using the shaped cost function $\hat c$ and dynamics $\hat\PP$:
        \[
            \overline{Q}(s,a) = \E_{\hat\rho^\mu | s_0 = s, a_0 = a} \left[ \sum_{t=0}^\infty \gamma^t \hat c(s_t, a_t) \right].
        \]
        For our experiments, the modeled mass $\hat m$ is either $1$ (unbiased case) or $0.5$ (biased case).

\end{itemize}
For our experiments, we set the advantage threshold $\eta = 0.08$ when using the neural network approximator and $\eta = 0$ when using the model-based rollouts.

\textbf{Hyperparameters:}
All point experiments were run on a 32-core Threadripper machine.
The given hyperparameters were found by hand-tuning until good performance was found on all algorithms.

\begin{table}[h!]
\begin{tabular}{|l|l|} \hline
{\bf Hyperparameter}                                             & {\bf Value}                                                 \\ \hline
Epochs                                                     & 500                                                   \\
Neural Network Architecture                                & 2 hidden layers, 64 neurons per hidden layer, $\tanh$ act. \\
Batch size                                                 & 4000                                                  \\
Discount $\gamma$                                          & 0.99                                                  \\
Entropy bonus                                              & 0.001                                                 \\
CMDP threshold $\delta$                                    & 0.01                                                  \\
Penalty value $\tilde R$                                   & $-2$                                                  \\
Lagrange multiplier step size (for constrained approaches) & 0.05                                                  \\
Cost shaping constant $\alpha$                             & 0.5                                                   \\
Number of seeds                                            & 10                                                    \\ \hline
\end{tabular}
\end{table}

\subsection{Half-Cheetah} \label{app:cheetah}
\begin{figure}[h!]
    \centering
    \includegraphics[width=0.3\textwidth]{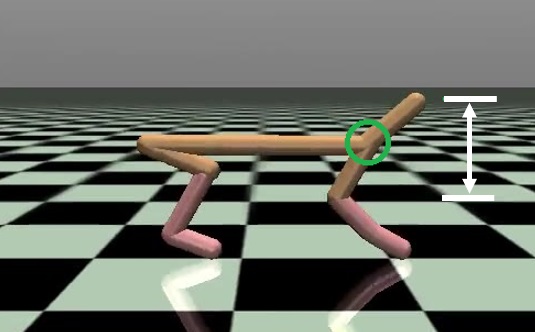}
    \caption{The half-cheetah environment.
    The green circle is centered on the link of interest, and the white double-headed arrow denotes the allowed height range of the link.}
    \label{fig:cheetah environment}
\end{figure}
This environment (\cref{fig:cheetah environment}) comes from OpenAI Gym and has reward equal to the agent's forward velocity.
One of the agent's links (denoted by the green circle in \cref{fig:cheetah environment}) is constrained to lie in a given height range, outside of which the robot is deemed to have fallen over.
In other words, if $h$ is the height of the link of interest, $h_{\min}$ is the minimum height, and $h_{\max}$ is the maximum height, the safe set is defined as $\safeset = \{s \in \SS : h_{\min} \le h \le h_{\max} \}$.
For our experiments, we set $h_{\min} = 0.4$ and $h_{\max} = 1$.

\textbf{Heuristic Intervention Rule:}
This intervention rule $\GG = (\overline{Q}, \mu, \eta)$ relies on a dynamics model (here, unbiased) to greedily predict whether the safety constraint would be violated at the next time step.
In particular, if $s$ is the current state and $\hat{a} \sim \pi(\cdot | s)$ is the proposed action, the agent will be intervened if the height $\hat{h}'$ in the next state $\hat{s}' \sim \PP(\cdot | s, \hat{a})$ lies outside the range $[\hat{h}_{\min}, \hat{h}_{\max}]$, where $\hat{h}_{\min}$ and $\hat{h}_{\max}$ can be set to a smaller range than $[h_{\min}, h_{\max}]$ to induce a more conservative intervention.
Once intervened, the episode terminates.
The reason for using a smaller range $[\hat{h}_{\min}, \hat{h}_{\max}]$ is an attempt to make the intervention rule possess the partial property (see the discussion in \cref{sec:gr}). If we were to set the range to be the ordinary range $[h_{\min}, h_{\max}]$ that defines the safe subset, the penalty $\widetilde{R}$ would need to be very negative, which would destabilize learning.
Furthermore, there is no guarantee that the intervention set for the original range is partial since there may be no available action to keep the agent from being intervened.

\textbf{MPC-Based Intervention Rule:}
Similarly with the model-based intervention rule for the point environment, the MPC intervention rule $\GG = (\overline{Q}, \mu, \eta)$ uses a model of the half-cheetah.
The backup policy $\mu$ is a sampling-based model predictive control (MPC) algorithm based on \citep{williams2017information}.
The MPC algorithm has an optimization horizon of $H = 16$ time steps and minimizes the cost function corresponding to an indicator function of the link height being in the range $[0.45, 0.95]$.\footnote{Observe that this is slightly smaller than the $[0.4, 1]$ height range of the original safety constraint.}
The function $\overline{Q}$ is defined as:
\[
    \overline{Q}(s,a) = \E_{\hat\rho} \left[\sum_{t=0}^H \gamma^t \hat{c}(\hat{s}_t, \hat{a}_t) \;\middle|\; \hat{s}_0 = s, \hat{a}_0 = a, \hat{a}_{1:H} = \mathrm{MPC}(\hat{s}_1) \right],
\]
where $\hat{c}(s, a)$ is the hinge-shaped cost function (in \eqref{eq:hinge}) corresponding to the distance function $\mathrm{dist}(s, S_\mathrm{unsafe}) = \max\{0, \min\{h - h_{\min}, h_{\max} - h\}\}$.

For our experiments, we set the advantage threshold $\eta = 0.2$.
We also use a modeled mass of $14$ (unbiased) and $12$ (biased) in our experiments.

{\bf Hyperparameters}: Except for the MPC-based intervention, all half-cheetah experiments were run on a 32-core Threadripper machine.
The MPC-based intervention experiments were run on 64-core Azure servers with each run taking 24 hours.
The given hyperparameters were found by hand-tuning until good performance was found on all algorithms.

\begin{table}[h!]
\begin{tabular}{|l|l|} \hline
{\bf Hyperparameter}                                             & {\bf Value}                                                 \\ \hline
Epochs                                                     & 1250                                                   \\
Neural Network Architecture                                & 2 hidden layers, 64 neurons per hidden layer, $\tanh$ act. \\
Batch size                                                 & 4000                                                  \\
Discount $\gamma$                                          & 0.99                                                  \\
Entropy bonus                                              & 0.01                                                 \\
CMDP threshold $\delta$                                    & 0.01                                                  \\
Penalty value $\tilde R$                                   & $-0.1$                                                  \\
Lagrange multiplier step size (for constrained approaches) & 0.05                                                  \\
Heuristic intervention range $[\hat{h}_{\min}, \hat{h}_{\max}]$ & $[0.4, 0.9]$                                     \\
Cost shaping constant $\alpha$                             & 0.05                                                   \\
Number of seeds                                            & 8                                                    \\ \hline
\end{tabular}
\end{table}

\section{Ablations for Point Robot}

\begin{figure}[h!]
    \centering
    {\bf Episode return without the intervention} \\
	\begin{subfigure}{0.3\textwidth}
		\includegraphics[width=\textwidth]{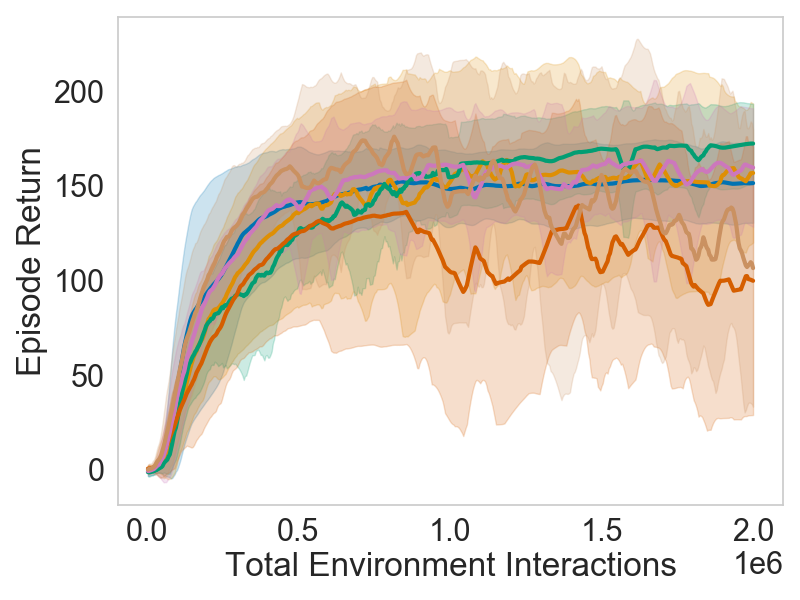}
	\end{subfigure}
	\begin{subfigure}{0.3\textwidth}
		\includegraphics[width=\textwidth]{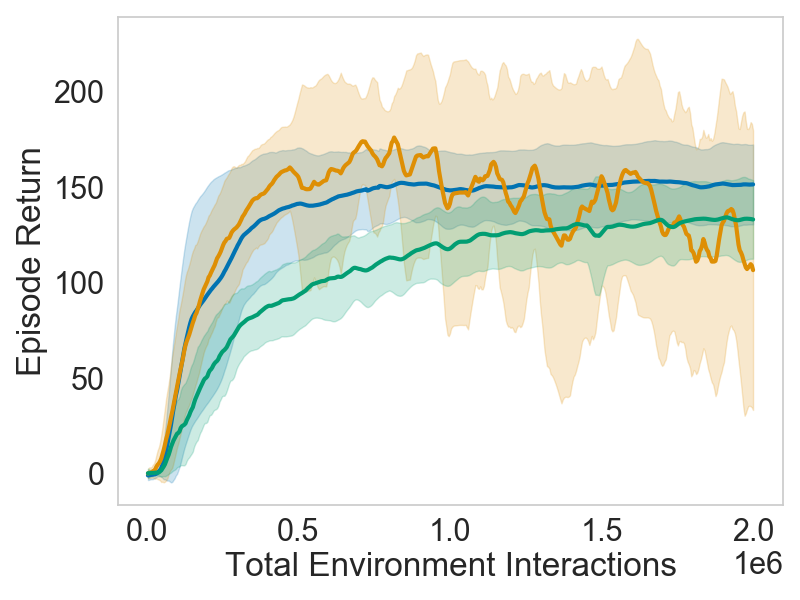}
	\end{subfigure}
	
	{\bf Episode length without the intervention} \\
	\begin{subfigure}{0.3\textwidth}
		\includegraphics[width=\textwidth]{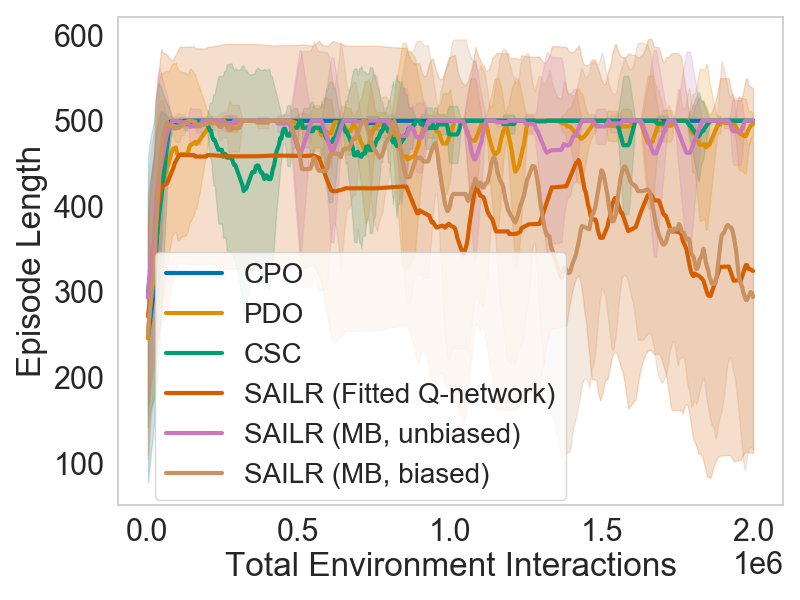}
	\end{subfigure}
	\begin{subfigure}{0.3\textwidth}
		\includegraphics[width=\textwidth]{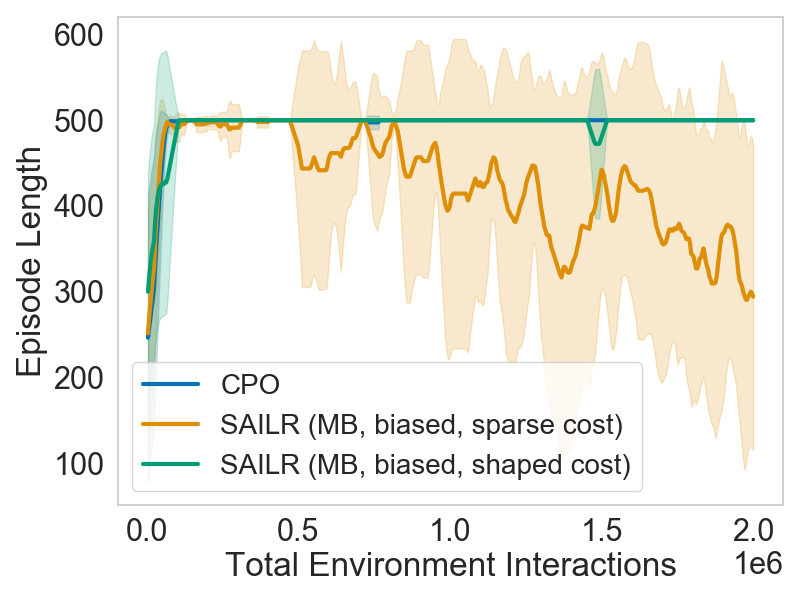}
	\end{subfigure}
	
	{\bf Total number of safety violations during training} \\
	\begin{subfigure}{0.3\textwidth}
		\includegraphics[width=\textwidth]{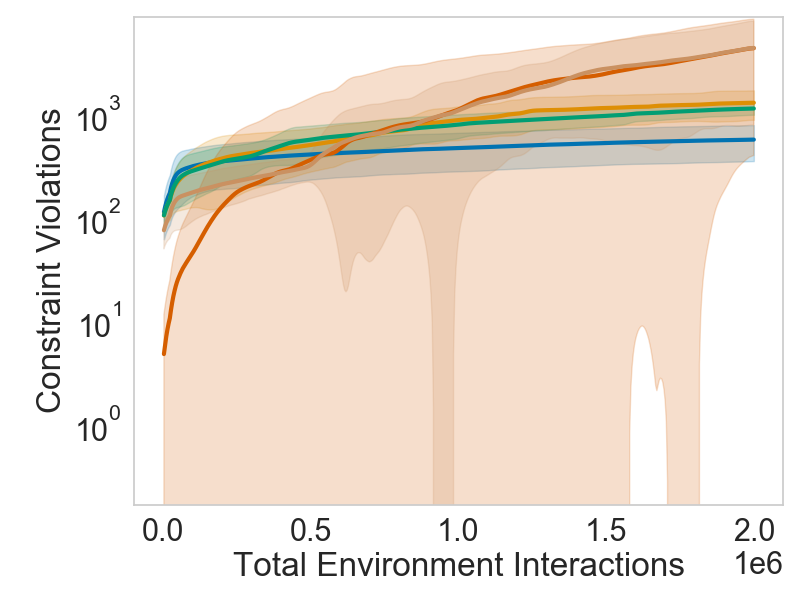}
		\caption{Sparse cost}
		\label{fig:point sparse cost}
	\end{subfigure}
	\begin{subfigure}{0.3\textwidth}
		\includegraphics[width=\textwidth]{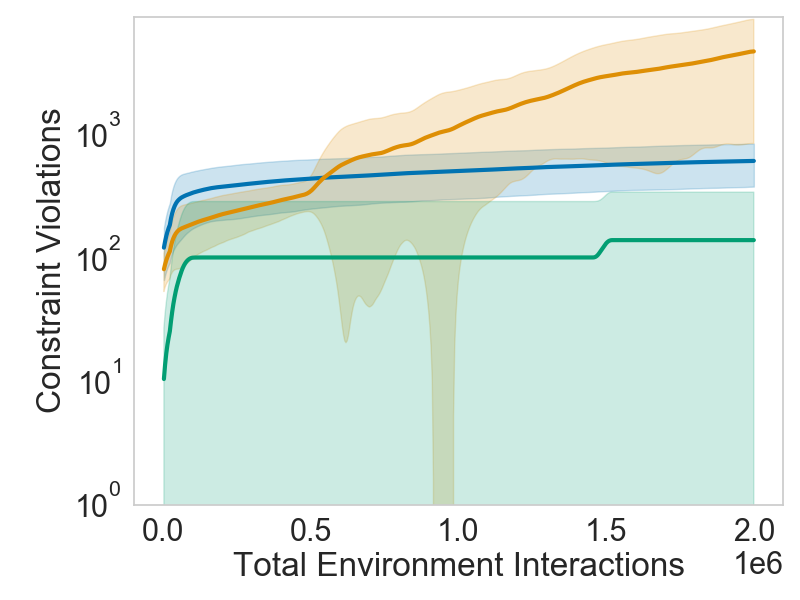}
		\caption{Biased model, either a sparse cost or shaped cost}
		\label{fig:point biased model}
	\end{subfigure}
\caption{Ablations for point experiment}
\label{fig:point ablations}
\end{figure}

We run the following two ablations for the point environment, with results shown in \cref{fig:point ablations}:
\begin{enumerate}
	\item We additionally run all the algorithms with the original sparse cost~(\cref{fig:point sparse cost}).
		Here, the baseline algorithms as expected yield high deployment returns while violating many constraints during training.
		For \alg, however, only the model-based instance with an unbiased model is able to satisfy the desiderata of high deployment returns while being safe during training.
		In this case, the sparse cost along with the approximation errors from the other two instances result in the slack $\sigma$ being large for admissibilty, meaning the safety bounds in~\cref{th:main theorem,th:shielded policy is safe} are loose.

	\item We run the model-based instance of \alg with a biased model and either the sparse cost or the shaped cost~(\cref{fig:point biased model}).
		Using the sparse cost with the biased model for intervention has deleterious effects in safety and performance.
		The model mismatch causes a compounding number of safety violations in training (bottom plot) and destabilizes the policy optimization, as observed in the deteriorated returns (top plot) and safety (middle plot), respectively.
		Shaping the cost function for intervention results in far fewer safety violations and stabilizes the policy optimization.
\end{enumerate}

\section{Varying Intervention Penalty for Half-Cheetah}

\begin{figure}[h!]
    \centering
    {\bf Episode return without the intervention} \\
	\begin{subfigure}{0.24\textwidth}
		\includegraphics[width=\textwidth]{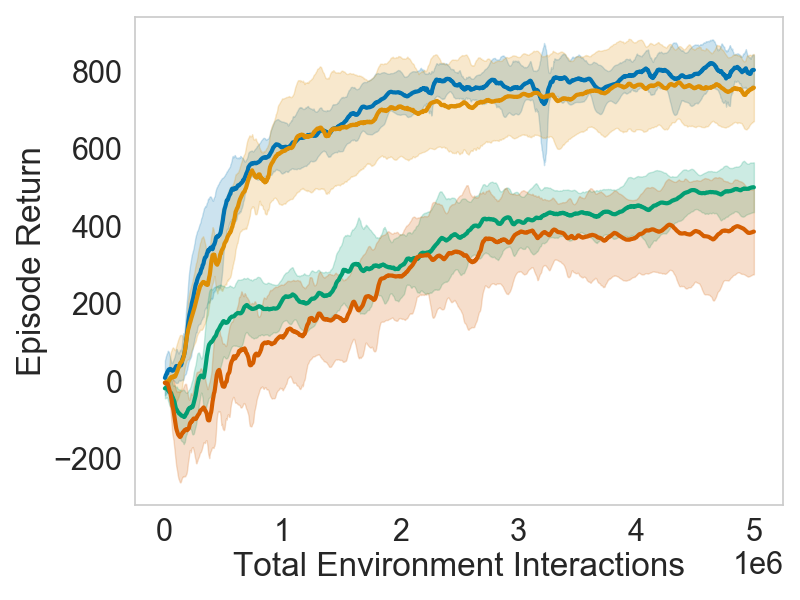}
	\end{subfigure}
	\begin{subfigure}{0.24\textwidth}
		\includegraphics[width=\textwidth]{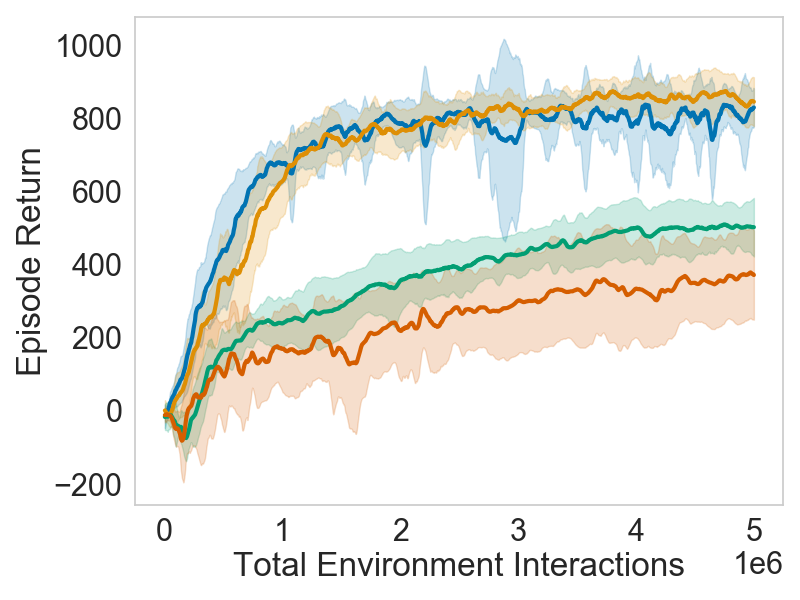}
	\end{subfigure}
	\begin{subfigure}{0.24\textwidth}
		\includegraphics[width=\textwidth]{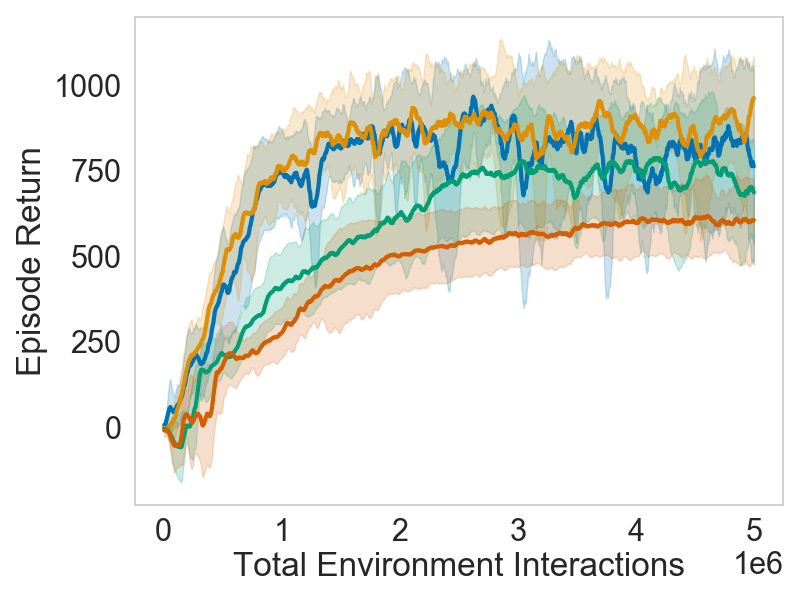}
	\end{subfigure}
	\begin{subfigure}{0.24\textwidth}
		\includegraphics[width=\textwidth]{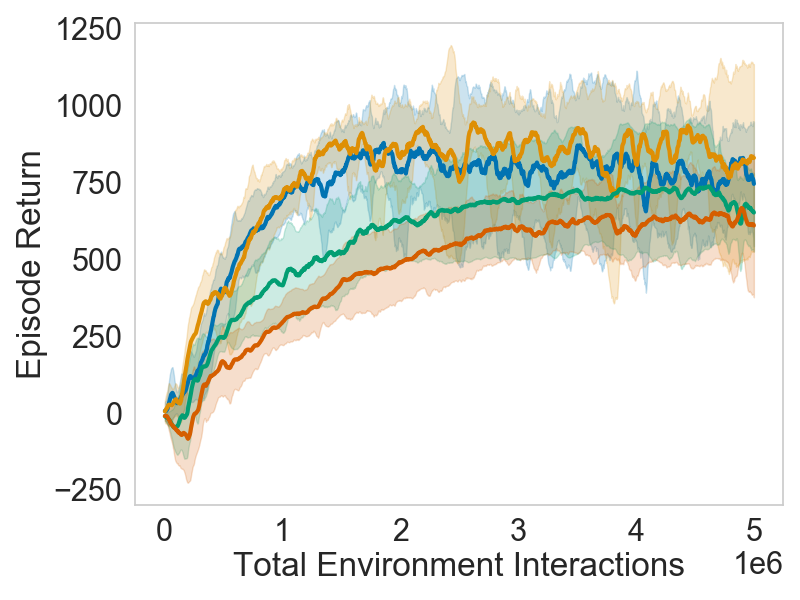}
	\end{subfigure}
	
	{\bf Episode length without the intervention} \\
	\begin{subfigure}{0.24\textwidth}
		\includegraphics[width=\textwidth]{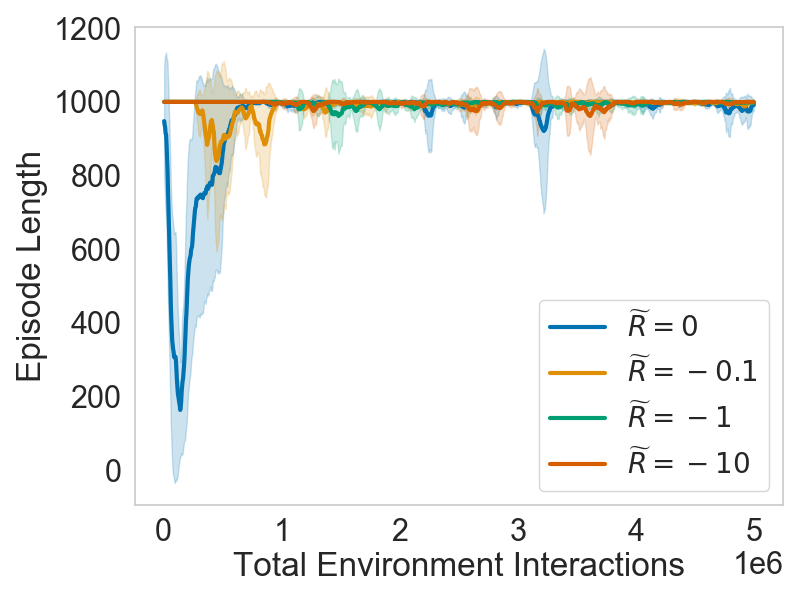}
	\end{subfigure}
	\begin{subfigure}{0.24\textwidth}
		\includegraphics[width=\textwidth]{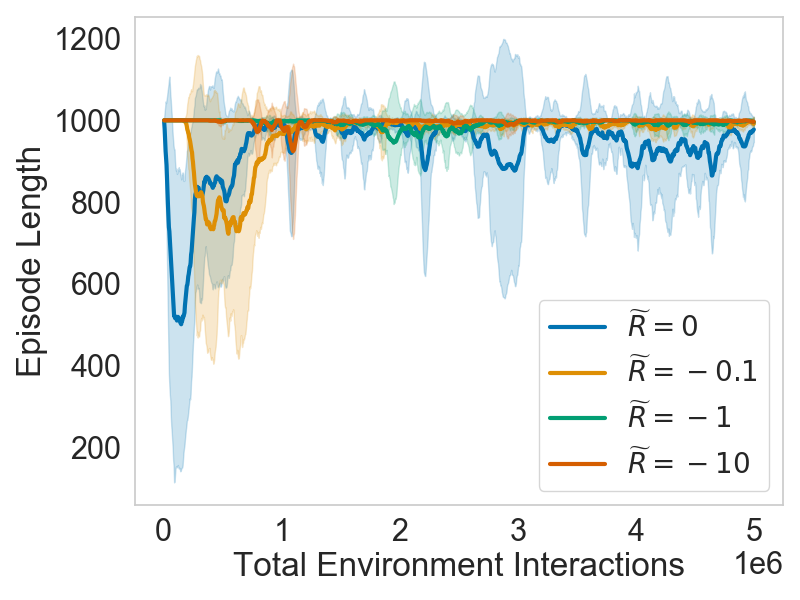}
	\end{subfigure}
	\begin{subfigure}{0.24\textwidth}
		\includegraphics[width=\textwidth]{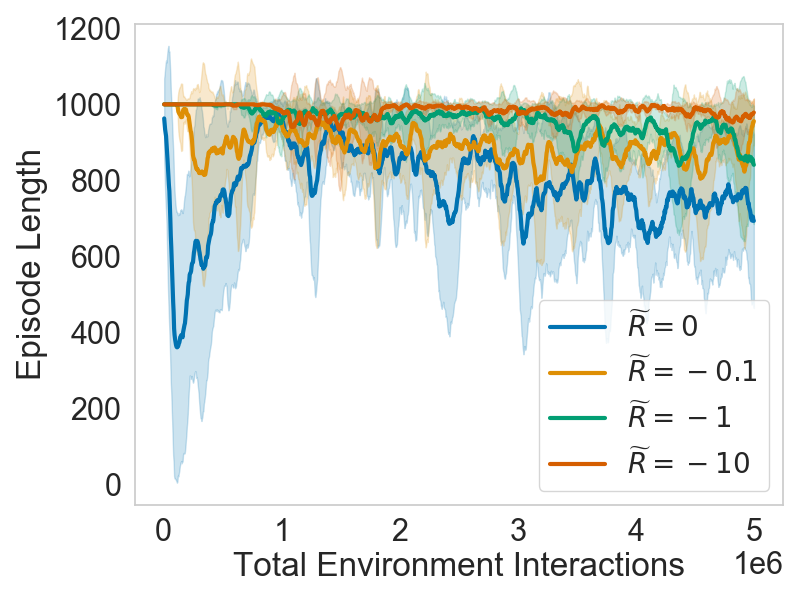}
	\end{subfigure}
	\begin{subfigure}{0.24\textwidth}
		\includegraphics[width=\textwidth]{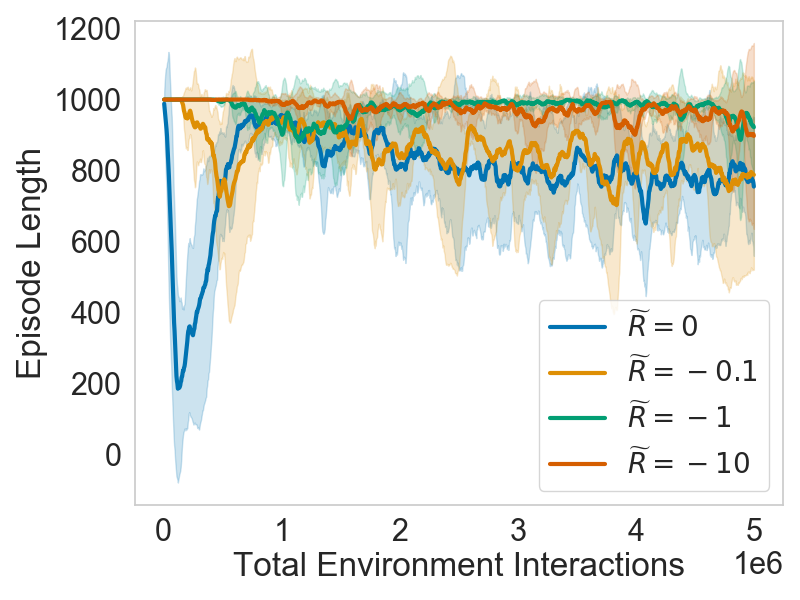}
	\end{subfigure}
	
	{\bf Total number of safety violations during training} \\
	\begin{subfigure}{0.24\textwidth}
		\includegraphics[width=\textwidth]{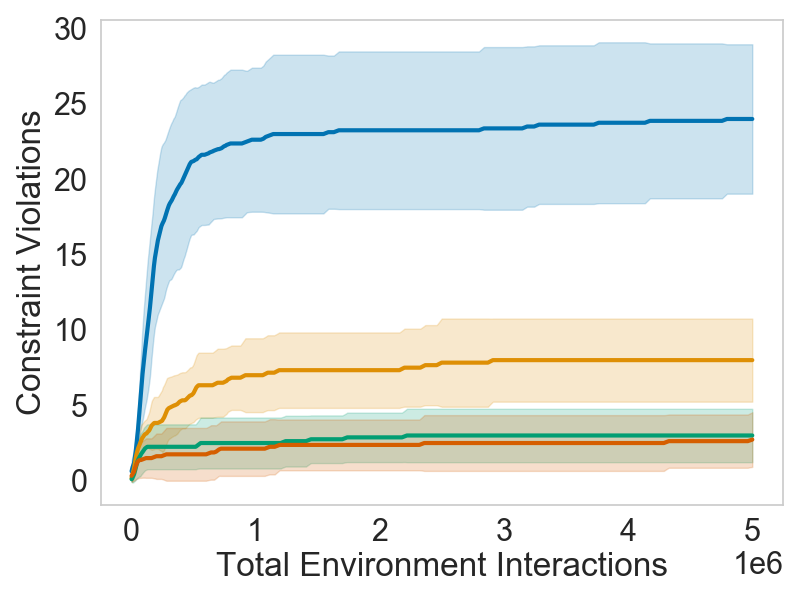}
		\caption{MPC with unbiased model}
		\label{fig:cheetah mpc unbiased}
	\end{subfigure}
	\begin{subfigure}{0.24\textwidth}
		\includegraphics[width=\textwidth]{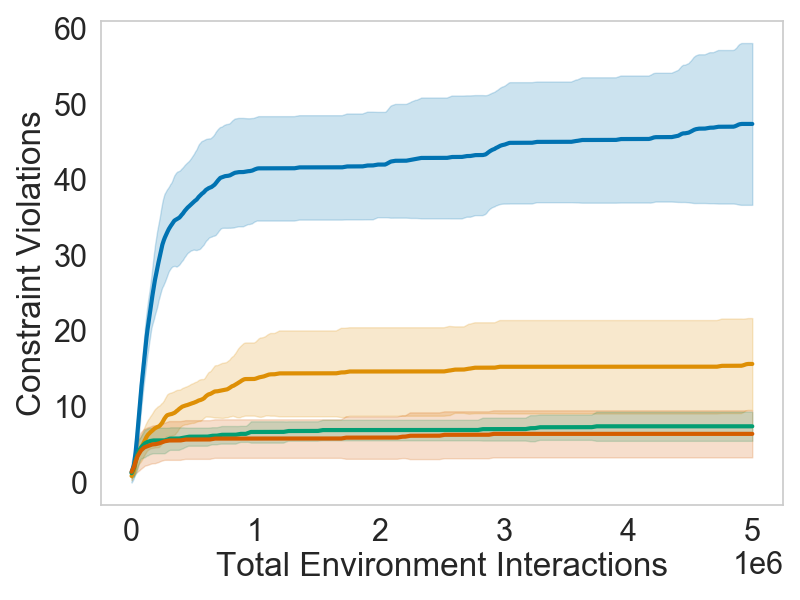}
		\caption{MPC with biased model}
		\label{fig:cheetah mpc biased}
	\end{subfigure}
	\begin{subfigure}{0.24\textwidth}
		\includegraphics[width=\textwidth]{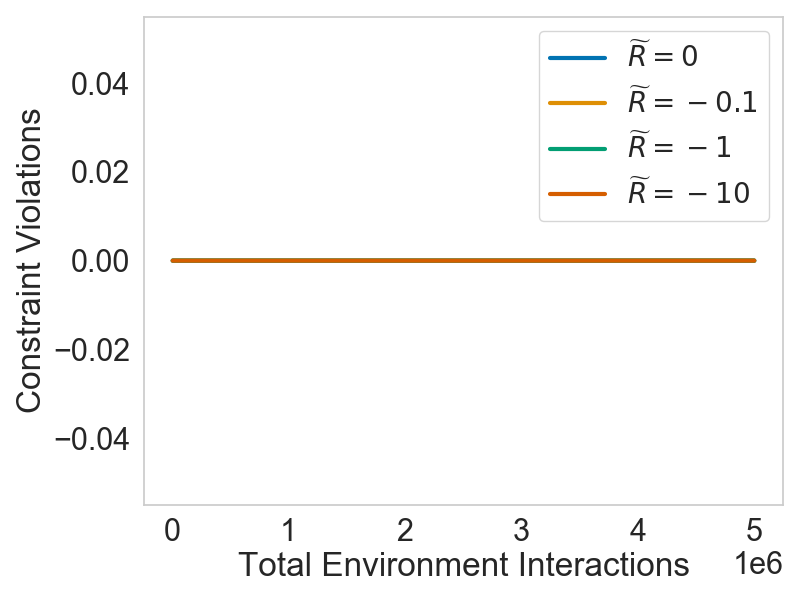}
		\caption{Heuristic with smaller height range}
		\label{fig:cheetah heuristic}
	\end{subfigure}
	\begin{subfigure}{0.24\textwidth}
		\includegraphics[width=\textwidth]{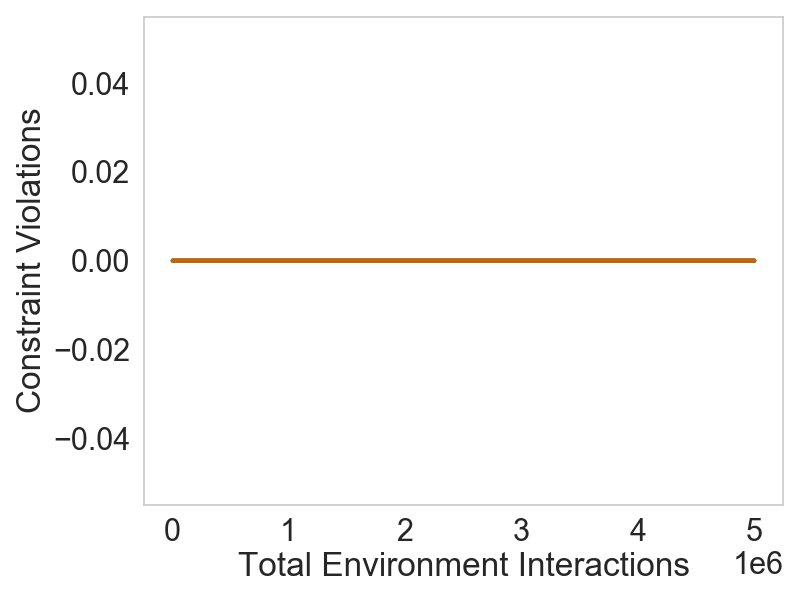}
		\caption{Heuristic with original height range}
		\label{fig:cheetah original heuristic}
	\end{subfigure}
\caption{Varying intervention penalty for half-cheetah experiment}
\label{fig:cheetah penalty}
\end{figure}

We vary the intervention penalty $\widetilde R$ for both the MPC-based intervention and heuristic intervention~(\cref{fig:cheetah penalty}).
Common among all results is that the deployment episode return (top row) decreases and deployment safety (middle row) increases with the magnitude of the penalty, consistent with the performance and safety bounds in~\cref{th:performance and safety}.
Furthermore, early in training, we remark that larger penalties result in the agent learning to be safe more quickly (middle row).
This is likely because the large penalties prioritize the agent to not be intervened, which allows it to more quickly learn to be as safe as the backup policy $\mu$.

For training-time safety with the MPC backup policy (bottom row of~\cref{fig:cheetah mpc unbiased,fig:cheetah mpc biased}), we observe that there are more violations as the penalty decreases, likely because the agent is less conservative during rollouts.

For the heuristic intervention~(\cref{fig:cheetah heuristic,fig:cheetah original heuristic}), we surmise that neither heuristic is partial since we require $\widetilde R$ to be relatively large in order for the agent to learn to be safe (middle row).
This is in constrast with the MPC-based intervention rule (middle row of~\cref{fig:cheetah mpc unbiased,fig:cheetah mpc biased}), where the penalty only needs to be nonzero, which indicates that the MPC-based intervention is partial.

\end{document}